\documentclass{tlp}

\usepackage{xcolor}
\usepackage{latexsym}
\usepackage{amssymb}
\usepackage{amsmath}
\usepackage{esvect}
\usepackage{xcolor}
\usepackage{xspace}
\usepackage{graphicx}
\usepackage{booktabs}
\usepackage{url}
\usepackage[utf8]{inputenc}
\usepackage{mathtools}
\usepackage{stmaryrd}
\usepackage{multirow}
\usepackage{tikz}
\usetikzlibrary{shapes,arrows,positioning,patterns,decorations.pathreplacing}
\usepackage{listings}
\usepackage{centernot}
\usepackage{stackengine}
\usepackage{scalerel}
\usepackage{pgffor}

\foreach \x in {A,...,Z}{%
\expandafter\xdef\csname cal\x\endcsname{\noexpand\ensuremath{\noexpand\mathcal{\x}}}
\expandafter\xdef\csname frak\x\endcsname{\noexpand\ensuremath{\noexpand\mathfrak{\x}}}
\expandafter\xdef\csname bb\x\endcsname{\noexpand\ensuremath{\noexpand\mathbb{\x}}}
\expandafter\xdef\csname sf\x\endcsname{\noexpand\ensuremath{\noexpand\mathsf{\x}}}
\expandafter\xdef\csname bf\x\endcsname{\noexpand\ensuremath{\noexpand\mathbf{\x}}}
}

\foreach \x in {a,...,z}{%
\expandafter\xdef\csname cal\x\endcsname{\noexpand\ensuremath{\noexpand\mathcal{\x}}}
\expandafter\xdef\csname frak\x\endcsname{\noexpand\ensuremath{\noexpand\mathfrak{\x}}}
\expandafter\xdef\csname bb\x\endcsname{\noexpand\ensuremath{\noexpand\mathbb{\x}}}
\expandafter\xdef\csname sf\x\endcsname{\noexpand\ensuremath{\noexpand\mathsf{\x}}}
}

\foreach \x in {cp,ME,id,KIR,OLEP,OSEP,Pred,fus,prop,Const,fo,as,doc,w,un}{%
\expandafter\xdef\csname \x\endcsname{\noexpand\ensuremath{\noexpand\mathsf{\x}}}
}

\foreach \x in {ALC}{%
\expandafter\xdef\csname \x\endcsname{\noexpand\ensuremath{\noexpand\mathcal{\x}}}
}

\foreach \x in {}{%
\expandafter\xdef\csname \x\endcsname{\noexpand\textbf{\x}}
}

\newcommand{\mref}[1]{\ensuremath{\mbox{\scriptsize{\ref{#1}}}}}

\newcommand{\etc}{etc.\xspace }
\newcommand{\resp}{resp.\xspace }

\newcommand{\ie}{i.\,e.\xspace }

\newcommand{\eg}{e.\,g.\xspace }

\newcommand{\viz}{viz.\xspace }
\newcommand{\wrt}{w.\,r.\,t.\xspace }

\newtheorem{definition}{Definition}
\newtheorem{example}{Example}
\newtheorem{theorem}{Theorem}
\newtheorem{proposition}{Proposition}
\newtheorem{corollary}{Corollary}
\newtheorem{lemma}{Lemma}

\newcommand{\head}[1][r]{\ensuremath{H(#1)}\xspace}
\newcommand{\bodyp}[2][r]{\ensuremath{B^{#2}(#1)}\xspace}
\newcommand{\body}[1][r]{\ensuremath{B(#1)}\xspace}
\newcommand{\bodyx}[1][r]{\ensuremath{B^*(#1)}\xspace}

\newcommand{\ol}[1]{\ensuremath{\overline{#1}}}
\newcommand{\rowprefix}[1]{%
  \ensuremath{#1{:}\ }%
}
\newcommand{\progprefix}[1]{%
  \ensuremath{#1{:}~~~}%
}
\DeclareMathOperator{\when}{\leftarrow}
\DeclareMathOperator{\nott}{\mathit{not}}

\newcommand{\nlit}[1]{\ensuremath{\ol{#1}}}
\newcommand{\pp}{\ensuremath{\calP}\xspace}

\newcommand{\unthicksim}{\mathord{\thicksim}}
\newcommand{\dcop}{\ensuremath{\unthicksim}}
\newcommand{\dc}[1]{\ensuremath{\dcop#1}}

\newcommand{\pAtoms}{\ensuremath{\calA}\xspace}
\newcommand{\LitA}{\ensuremath{Lit_\pAtoms}\xspace}
\newcommand{\Litx}{\ensuremath{Lit^\ast}\xspace}
\newcommand{\LitxA}{\ensuremath{\Litx_\pAtoms}\xspace}
\newcommand{\LitExt}{\ensuremath{L^*}\xspace}

\newcommand{\LitExtPrime}{\ensuremath{K^*}\xspace}

\newcommand{\nottt}{{\thicksim}}
\newcommand{\notttt}{{\sim}}

\newcommand{\inputpp}{\ensuremath{\pp_{\dbE}}\xspace}
\newcommand{\corepp}{\ensuremath{\pp_{\dbI}}\xspace}

\newcommand{\atom}[1]{\ensuremath{atom(#1)}}
\newcommand{\Atom}[1]{\ensuremath{Atom(#1)}}
\newcommand{\Lxall}[1]{\ensuremath{\calL^*({#1})}\xspace}
\newcommand{\litset}{\ensuremath{S}\xspace}
\newcommand{\litsett}{\ensuremath{T}\xspace}
\newcommand{\litsetx}{\ensuremath{X}\xspace}
\newcommand{\dbI}{\ensuremath{\calC}\xspace}
\newcommand{\dbE}{\ensuremath{\calF}\xspace}

\newcommand{\ans}[1]{\ensuremath{AS(#1)}\xspace}
\newcommand{\aspp}{\ensuremath{\ans{\pp}}\xspace}
\newcommand{\ppi}[1]{\ensuremath{\pp^{(#1)}}\xspace}

\newcommand{\Confl}[1]{\ensuremath{\mathit{Conflicts}({#1})}}
\newcommand{\confl}[1]{\ensuremath{\mathit{conflicts}({#1})}}
\DeclareMathOperator{\conf}{\bowtie}
\DeclareMathOperator{\nconf}{\centernot\bowtie}

\newcommand{\conflr}[1]{\ensuremath{Adv(#1)}}
\newcommand{\pc}{\ensuremath{r \conf r'\text{ (possibly)}}}
\newcommand{\nc}{\ensuremath{r \nconf r'}}
\newcommand{\blanket}{\ensuremath{\beta}\xspace}
\newcommand{\lambdasol}{\ensuremath{\lambda}\xspace}
\newcommand{\lambdasolrule}[1]{\ensuremath{\lambda(#1)}\xspace}

\newcommand{\negLxOp}{\ensuremath{neg}\xspace}
\newcommand{\NegLxOp}{\ensuremath{Neg}\xspace}
\newcommand{\NegLXOp}{\ensuremath{N}\xspace}
\newcommand{\negLx}[1]{\ensuremath{\negLxOp(#1)}\xspace}
\newcommand{\NegLx}[1]{\ensuremath{\NegLxOp(#1)}\xspace}
\newcommand{\NegLX}[1]{\ensuremath{\NegLXOp(#1)}\xspace}
\newcommand{\negLxdep}[2]{\ensuremath{\negLxOp_{#2}(#1)}\xspace}
\newcommand{\NegLXminOp}{\ensuremath{\NegLXOp_{min}}\xspace}
\newcommand{\NegLXmin}[1]{\ensuremath{\NegLXminOp(#1)}\xspace}
\newcommand{\blanketop}[1]{\ensuremath{blankets(#1)}\xspace}
\newcommand{\advbodies}[1]{\ensuremath{\calA\calB(#1)}\xspace}
\newcommand{\extAtoms}{\ensuremath{\calE_{\pp}}\xspace}
\newcommand{\LitE}{\ensuremath{Lit\extAtoms}\xspace}

\newcommand{\rulset}{\ensuremath{\calR}\xspace}

\newcommand{\ppgen}{\ensuremath{\pp^{+}}\xspace}
\newcommand{\arop}{\ensuremath{\mathit{AR}}}
\newcommand{\ar}[1]{\ensuremath{\arop(#1)}\xspace}
\newcommand{\eff}[2]{\ensuremath{\mathit{Eff}(#1,#2)}\xspace}

\newcommand{\exLabel}[1]{\ensuremath{\mathit{#1}}\xspace}

\newcommand{\exEligX}{\exLabel{\mathit{eligX}}}
\newcommand{\exCond}[1]{\exLabel{cond{#1}}\xspace}
\newcommand{\exCondA}{\exCond{A}}
\newcommand{\exCondAAdv}{\exCondA{Adv}}

\newcommand{\exPreTreated}[1]{\exLabel{preTreated{#1}}}
\newcommand{\exPreTreatedM}{\exPreTreated{M}}
\newcommand{\exPreTreatedN}{\exPreTreated{N}}
\newcommand{\exHighLCount}{\exLabel{highLCount}}

\newcommand{\exDrug}[1]{\exLabel{drug{#1}}}

\newcommand{\exCiR}{\exLabel{ctrIndR}}
\newcommand{\exCiS}{\exLabel{ctrIndS}}

\newcommand{\exDrugB}{\exDrug{B}}
\newcommand{\exDrugC}{\exDrug{C}}
\newcommand{\exDrugD}{\exDrug{D}}
\newcommand{\exDrugE}{\exDrug{E}}

\begin{document}

\lefttitle{Thevapalan and Kern-Isberner}

\jnlPage{1}{8}
\jnlDoiYr{20xx}
\doival{10.1017/xxxxx}

\title[On Establishing Robust Consistency in Answer Set Programs]
{On Establishing Robust Consistency in Answer Set Programs}

\begin{authgrp}
  \author{\gn{Andre} \sn{Thevapalan}}
  \affiliation{Technische Universität Dortmund, Dortmund, Germany}
  \author{\gn{Gabriele} \sn{Kern-Isberner}}
  \affiliation{Technische Universität Dortmund, Dortmund, Germany}
\end{authgrp}

\maketitle

\begin{abstract}
  Answer set programs used in real-world applications often require that the program is usable with different input data.
  This, however, can often lead to contradictory statements and consequently to an inconsistent program.
  Causes for potential contradictions in a program are conflicting rules.
  In this paper, we show how to ensure that a program \pp remains non-contradictory given any allowed set of such input data.
  For that, we introduce the notion of conflict-resolving \lambdasol-extensions.
  A conflict-resolving \lambdasol-extension for a conflicting rule $r$ is a set $\lambdasol$ of (default) literals such that extending the body of $r$ by $\lambdasol$ resolves all conflicts of $r$ at once.
  We investigate the properties that suitable \lambdasol-extensions should possess and building on that, we develop a strategy to compute all such conflict-resolving \lambdasol-extensions for each conflicting rule in \pp.
  We show that by implementing a conflict resolution process that successively resolves conflicts using \lambdasol-extensions eventually yields a program that remains non-contradictory given any allowed set of input data.

  \paragraph{Note:} This paper is under consideration for acceptance in Theory and Practice of Logic Programming (TPLP).
\end{abstract}

\begin{keywords}
  Logic Programming, Answer Set Programming, Consistency, Contradictions, Conflicts, Interactions
\end{keywords}

\section{Introduction}%
\label{sec:introduction}

\subsection{Motivation and Context}%
\label{sub:motivation_and_context}
  Answer set programs can be used to implement real-world applications like decision support systems that aid knowledge experts whenever crucial decisions based on different rules and conditions have to be made.
However, the knowledge bases of such applications are not static but rather very dynamic in the sense that they are adapted to each individual case and can also be prone to various updates.
Especially in domains like the medical sector, knowledge bases are expected to yield suitable decisions for each patient where patients can show diverse symptoms (implemented by facts), and the knowledge base is to be updated very often, \eg, every month.
In addition, physicians are usually faced with some degree of uncertainty and for example have to rely on own experience when they have to reach important decisions (\eg regarding the treatment of patients)~\citep{Ghosh2004}.
Imagine a system that outputs the possible treatment plans for a given patient based on a corresponding medical ruleset.
For each patient, the application has to combine the general knowledge about possible treatment plans (\emph{problem encoding}) with the data regarding the patient (\emph{problem instance})~\citep{GebserKaminskiKaufmannSchaub2012}.
Furthermore, the knowledge base containing the encoding can grow over time, not only adding new knowledge but also revising or removing deprecated knowledge.
  Additionally, decision-making processes in this sector do not only consist of, \eg, finding the right therapy for each patient, it is also important to assure that unfitting solutions are also reflected as such (\eg what therapies must not be recommended to a patient).
This emphasizes how the usage of an answer set program with both default and strong negation constitutes a highly valuable asset in assisting medical experts.

However, at any point where such an application is used, it has to be ensured that the respective knowledge base remains consistent when merged with the patient data.
For a knowledge base that is frequently updated and that also has to be used with different instance data, maintaining consistency can become a quite cumbersome task as it requires not only a complete understanding of the whole problem encoding but also technical knowledge regarding logic programming.
In previous approaches regarding the update of answer set programs~\citep{Eiter2002,AlferesLeitePereiraPrzymusinskaPrzymusinski1998}, every update required an automated adaptation of the updated program in order to prevent conflicts.
Using these methods for a medical support system would mean that the changes that are made in the new program are not managed and approved by the knowledge expert.
Naturally, these inconveniences can deter knowledge experts from using answer set programs for the implementation even though answer set programming itself is well suited for highly complex decision making problems.

It is for this reason that we propose a general framework that allows the maintenance of a logic program by involving the expertise of knowledge experts more directly.
Such maintenance tasks would include the detection of possible conflicting statements between the problem encoding and any possible instance data and the resolution of such conflicts.
Note that the consistency of an answer set program crucially depends on its facts, in particular, its input data.
Conflicts between rules may become apparent only if specific input data is provided, resulting in an unexpected failure in a specific (rare) case of an otherwise helpful and approved program.
Our approach aims at anticipating such conflicts, ensuring that a program yields professionally adequate solutions for any (future) case.
The resolution of such conflicts should be overseen by the knowledge expert and executed in an interactive fashion.
To facilitate the conflict resolution, possible solutions for each conflict should be generated and presented to the expert on demand.
They can then choose a most suitable solution.
In this way it is guaranteed that every modification that is applied to the knowledge base is valid in a professional sense.
Such an interactive exchange between the system and the expert during the resolution process can therefore eliminate the need for a deeper technical understanding of logic programming, giving the knowledge expert full control over the maintenance operations, guaranteeing that no changes are done blindly.
Especially in fields like the medical sector, full expert control over the knowledge base is crucial as small mistakes can have serious consequences.
The gap between a professional and a technical expert can be significant and make the automatic resolution of conflicts by suitable modifications of the program impossible. 
So, extensive communications between the two experts would be needed.  
The proposed framework constitutes a first step towards closing this gap as the conflict resolution process can be fully executed by the professional expert directly because all proposed modifications will be technically viable solutions of the conflict, and the expert can then choose from them according to professional standards.
Being able to establish the consistency of the problem encoding for every possible instance data without actually requiring the explicit instance has the effect that in practice, the program does not have to be validated for each instance, and makes it much more robust and reliable.

\subsection{Main Contributions}%
\label{sub:main_contributions}
In order to facilitate the integration of answer set programming in real-world applications, in this paper, we show how to ensure that the problem encoding in a logic program remains consistent given any (allowed) set of consistent instance data.
Inspired by~\citep{Eiter2002} and~\citep{AlferesLeitePereiraPrzymusinskaPrzymusinski1998,AlferesLeitePereiraPrzymusinskaPrzymusinski2000} and their work on logic program updates, we develop a strategy to extend rules of a logic program such that the derivation of contradictory statements is no longer possible in any admissible case.
The presented approach expands previous approaches in two major ways:

(1) Programs that are modified following our approach remain consistent for any given input data that do not contain atoms that appear in head literals of the program.
For that, we extend the notion of literals and investigate the relationship between the different negated versions of a literal.
This enables us to define the characteristics of conflicting and non-conficting rules, and in particular to determine how the body of a rule $r_i$ that is in conflict with several rules $r_j$ can be extended such that the modified rule $r'_i$ is no longer conflicting with any rule $r_j$ while the extension comprises only informative literals.
By informative, we mean that these extensions should only be composed of literals whose atoms appear in the body of the rules involved in the conflict.
For that reason, we also adapt the notion of \emph{hitting sets} for (default) literals in logic program rules.
As a result, the meaningfulness of the problem encoding is validated beforehand, eliminating the need for testing the encoding against all possible instances or other precautionary measures for each instance.

(2) Instead of using a technical device to prioritize one statement over another in case of a conflict like \emph{causal rejection \citep{Eiter2002}}, we show how all \emph{informative} extensions for a conflicting rule can be computed.
As a consequence, a knowledge expert could then be included into the resolution process and choose the suggested rule modification that is most suitable such that the resulting modified program remains professionally adequate.
Using informative extensions also maintains the readability of the rules in a program which in turn can simplify subsequent update operations.

  (3) The final goal is the construction of a  framework for the interactive maintenance of large ASP knowledge bases which can be used by knowledge experts without the need for technical knowledge and where each modification can be overseen by the expert.
  The approach presented in this paper is a basic building block for such a framework as the computation of all possible informative extensions can be used to resolve each conflict in cooperation with the knowledge expert which in turn guarantees that the resulting knowledge base still contains professionally adequate knowledge.
  We will also show that the approach follows a pragmatic paradigm which allows various extensions in order to facilitate the conflict resolution process for the expert.

\subsection{Structure of the Paper}%
\label{sub:structure_of_the_paper}
The paper is organized as follows:
We begin by providing the necessary preliminaries about extended logic programs in Section~\ref{sec:preliminaries}.
In Section~\ref{sec:contradictions}, we present our conflict resolution approach.
We start in Section~\ref{subsec:conflicts} by introducing the notion of uniformly non-contradictory program cores and examine the properties of conflicts and non-conflicting rules in Section~\ref{subsec:conflict_detection} before we connect these results in Section~\ref{subsec:conflict_resolution} by defining the properties of a conflict resolution step.
After describing  a na\"ive conflict resolution approach in Section~\ref{ssec:semi-normal-completion} by using semi-normal completions~\citep{Caminada2006}, we present in Section~\ref{subsec:conflict_resolution_with_informative_extensions} our strategy to compute all appropriate rule modifications of conflicting rules in order to obtain a uniformly non-contradictory program core.
In Section~\ref{sec:extensions}, we present additional ways to use and extend the conflict resolution approach for the usage in real-world applications.
Sections~\ref{subsec:m_to_n_conflicts} and~\ref{ssec:conflict_resolution_with_constraints} show how the presented method can be used to resolve many-to-many conflicts and how inconsistency that can be caused by a specific type of constraints can be prevented.
Sections~\ref{subsec:conflict_order} and~\ref{subsec:lambdasol_scores} outline ways to enhance the handling of \lambdasol-extensions in order to facilitate their usage in applications.
The paper concludes with a summary and a discussion regarding future work.

\section{Preliminaries}%
\label{sec:preliminaries}

\subsection{Extended Logic Programs}%
\label{subsec:extended_logic_programs}
In this paper, we look at non-disjunctive \emph{extended logic programs} (ELPs) \citep{Gelfond1991}.
An ELP is a finite set of rules over a set $\calA$ of propositional atoms.
First, we discuss the different forms of negation in ELPs and introduce notations.
A classical literal $L$ is either an atom $A$ (\emph{positive literal}) or a negated atom $\neg A$ (\emph{negative literal}).
For a literal $L$, the \emph{strongly complementary} literal $\nlit{L}$ is $\neg A$ if ${L = A}$ and $A$ otherwise.
For a set $\litset$ of classical literals, ${\nlit{\litset} = \{\nlit{L} \mid L \in \litset\}}$ is the set of corresponding strongly complementary literals.
Then, $\LitA$ denotes the set $\pAtoms \cup \nlit\pAtoms$ of all classical literals over $\pAtoms$.
A \emph{default-negated literal} $L$, called \emph{default literal}, is written as $\nottt L$.
In logic programs, $\nottt$ will be used as a prefix solely for classical literals, symbolyzing the default negation which is usually denoted by the prefix $\nott$ \citep{Gelfond1991}.
Outside of logic programs, we will use $\nottt$ as a unary junctor in order to describe the \emph{default complement} of a (default) literal $(\nottt)L$, \ie, the default complement of $L$ is $\nottt L$ and the default complement of $\nottt L$ is $\nottt \nottt L = L$.
This reflects the binary characteristic of the default negation which is illustrated in Figure~\ref{fig:negation}.
For a set $\litset$ of classical literals, we define $\dcop$ accordingly, \ie, ${\dc{\litset} = \{\dc{L} \mid  L\in \litset \}}$.
Given a set \litset of (classical) literals, we say a (classical) literal \emph{$L$ is true in \litset} (symbolically ${\litset \vDash L}$) iff ${L \in \litset}$ and \emph{$\nottt L$} is true in \litset (symbolically ${\litset \vDash \nottt L}$) iff ${L \notin \litset}$.
By an \emph{extended literal} \LitExt, we either mean a (classical) literal $L$ or a default-negated literal $\nottt L$.
The set of all extended literals over a set of atoms \pAtoms will be denoted by \LitxA, \ie, ${\LitxA = \LitA \cup \nottt \LitA}$.
A set \litsetx of extended literals is true in \litset  (symbolically $\litset \vDash \litsetx$) iff every extended literal ${\LitExt \in \litsetx}$ is true in \litset.
With ${\atom{\LitExt}}$, we will associate the atom on which the extended literal \LitExt is based on.
The underlying atoms of a set $X$ of extended literals is given by the set of atoms $\Atom{X} = \{\atom{\LitExt} \mid \LitExt \in X \}$.
Two extended literals \LitExt,\LitExtPrime are \emph{atom-related} if $\atom{\LitExt} = \atom{\LitExtPrime}$.

\begin{figure}
  \begin{tikzpicture}
    \draw (0,2.25) rectangle (1,2.75) node[pos=.5] {$true$};
    \draw (1,2.25) rectangle (2,2.75) node[pos=.5] {$undec$};
    \draw (2,2.25) rectangle (3,2.75) node[pos=.5] {$false$};

    \draw (0,1.5) rectangle (1,2) node[pos=.5] {$false$};
    \draw (1,1.5) rectangle (3,2) node[pos=.5] {$true$};

    \draw (0,.75) rectangle (2,1.25) node[pos=.5] {$true$};
    \draw (2,.75) rectangle (3,1.25) node[pos=.5] {$false$};

    \draw (0,0) rectangle (1,0.5) node[pos=.5] {$false$};
    \draw (1,0) rectangle (2,0.5) node[pos=.5] {$undec$};
    \draw (2,0) rectangle (3,0.5) node[pos=.5] {$true$};

    \draw (0,2.25) -- (0,2.75) node[anchor=east,pos=.5] {$A$};
    \draw (0,1.5) -- (0,2) node[anchor=east,pos=.5] {$\nottt A$};
    \draw (0,.75) -- (0,1.25) node[anchor=east,pos=.5] {$\nottt \nlit A$};
    \draw (0,0) -- (0,0.5) node[anchor=east,pos=.5] {$\nlit A$};
  \end{tikzpicture}
  \caption{Visualization of truth values of an atom A \wrt a set of literals $S$}\label{fig:negation}
\end{figure}

The following definition introduces handy terms to describe the relationships between literals.

\begin{definition}\label{def:ext-lit-combo-types}
  Given a classical literal $L$, we say that
  \begin{itemize}
    \item $L$ and $\nottt L$ are \emph{default complementary},
    \item $L$ and $ \nottt \nlit L$ are \emph{dual}, and
    \item $\nottt L$ and $\nottt \nlit L$ are \emph{reconcilable}.
  \end{itemize}
  Given two atom-related literals ${\LitExt,\LitExtPrime}$ with $\LitExt \neq \LitExtPrime$, we will say $\LitExt$ and $\LitExtPrime$ are
  \begin{itemize}
    \item \emph{complementary} if \LitExt, \LitExtPrime are strongly or default complementary, and
    \item \emph{compatible} if \LitExt, \LitExtPrime are reconcilable or dual.
  \end{itemize}
\end{definition}

A set of classical literals is \emph{inconsistent} if it contains strongly complementary literals.
A set of extended literals is \emph{inconsistent} if it contains complementary literals.

Figure~\ref{fig:atom-related} visualizes the different relationships that two atom-related literals \LitExt,\LitExtPrime can have.
The different negation types and the particular significance of reconcilable literals will be discussed in Section~\ref{subsec:conflict_resolution}.

\begin{figure}
  \begin{tikzpicture}[sibling distance=10em,
    every node/.style = {shape=rectangle,
    draw, align=center},
    level 1/.style={sibling distance=7cm,level distance=1.4cm},
    level 2/.style={sibling distance=3cm}
    ]
    \node {atom-related \\$\LitExt \neq \LitExtPrime$}
      child {node {complementary}
        child {node {strongly\\complementary\\$L,\nlit L$}}
      child {node {default-\\complementary\\$L,\nottt L$}}}
      child {node {compatible}
        child {node {reconcilable\\ $\nottt L,\notttt \nlit L$}}
      child {node {dual\\$L,\nottt \nlit L$}}};
  \end{tikzpicture}
  \caption{Possible reationships between two atom-related literals}%
  \label{fig:atom-related}
\end{figure}
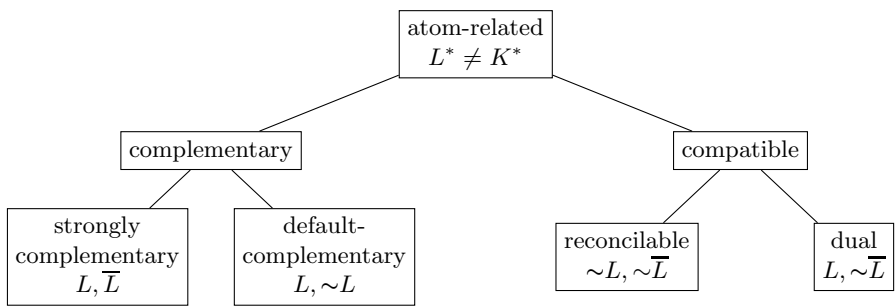
We are now ready to specify the form of ELPs.

A \emph{rule} $r$ is of the form
\begin{align}\label{eqn:rule}
  L_0 \when L_1, \dots, L_m, \nottt L_{m+1}, \dots, \nottt L_n.,
\end{align}
with classical literals $L_0, \dots, L_n$ and \mbox{$0 \leq m \leq n$}.
The literal $L_0$ is the \emph{head} of $r$, denoted by $\head$, and $\{L_1, \dots L_m, \nottt L_{m+1}, \dots \nottt L_n\}$ is the \emph{body} of $r$, denoted by \body.
Furthermore, $\{L_1, \dots, L_m\}$ is denoted by $\bodyp{+}$ and $\{L_{m+1}, \dots, L_n\}$ by $\bodyp{-}$.
Given a set of rules $\rulset \subseteq \pp$, we will denote the set of all extended literals occurring in the rule bodies of \rulset by $\bodyx[\rulset]$, \ie, $\bodyx[\rulset] = \bigcup_{r \in \rulset} \body$.
A rule $r$ with $\body = \emptyset$ is called a \emph{fact}, and $r$ is called a \emph{constraint} if it has an empty head.
For a fact $L.$, we will call the corresponding literal $L$ a \emph{fact literal}.
A set \dbE of facts is \emph{consistent} if the set of fact literals in \dbE is consistent.
A rule $r$ will be called a \emph{complex rule} if $r$ is neither a fact nor a constraint.
An \emph{extended logic program (ELP)} is a set of rules of the form (\ref{eqn:rule}).

The program \ppgen will denote the reduction of \pp to a \emph{normal program} (a logic program without classically negated literals) that is obtained by replacing every classically negated literal $\nlit A$ in \pp by a new corresponding atom $A'$ \citep{Gelfond1991}.

A rule $r$ is \emph{applicable} iff there is a consistent set \litset of classical literals such that $\litset \vDash \body$.
For the rest of this paper, we will assume that every rule in a given logic program is applicable.
Given a set \litset of classical literals, a (complex) rule $r$ \emph{is true in \litset} (symbolically ${\litset \vDash r}$) iff \head is true in \litset whenever \body is true in \litset.
In case $r$ is a constraint, $r$ is true in \litset iff $\litset \not\vDash$ \body.
Whenever a rule $r$ is true in \litset we also say that \litset \emph{satisfies} $r$.
A rule body \body will be called \emph{satisfiable} if there exists a consistent set $\litset$ of classical literals such that $\litset \vDash \body$.
Correspondingly, given a set of rules  $\rulset = \{r_1,\dots,r_n\}$, the bodies of all rules in \rulset are \emph{simultaneously satisfiable} whenever there is a set \litset of classical literals such that $\litset \vDash \body$ for every rule $r \in \rulset$.
Given an ELP \pp without default negation, the \emph{answer set} of \pp is
either (a) the smallest set $\litset \subseteq \LitA$ such that \litset is consistent and $\litset \vDash r$ for every rule $r \in \pp$,
or (b) the set \LitA of classical literals.
Note that similar to Horn logic programs, each such ELP has exactly one minimal model which might, however, be inconsistent.

In general, an \emph{answer set} of an ELP \pp is determined by its reduct.
The \emph{reduct} $\pp^\litset$ of a program \pp relative to a set $\litset$ of classical literals is defined by

\begin{align}
  \pp^\litset = \{H(r) \when \bodyp{+}. \mid r \in \pp, \bodyp{-} \cap \litset = \emptyset \}.
\end{align}

A set \litset of classical literals is an \emph{answer set} of \pp if it is the answer set of $\pp^\litset$ \citep{Gelfond1991}.
The set of all answer sets of a program \pp will be denoted by $AS(\pp)$.
Note that since every inconsistent set of literals that satisfies all rules in \pp will be replaced by \LitA, every reduct of \pp is formed relative to either a consistent set \litset or \LitA.
As a consequence, \aspp can either consist of only consistent sets of literals, only \LitA, or no sets at all.
We say a classical literal $L$ is \emph{derivable} in \pp iff ${L \in \bigcup AS(\pp)}$.
For a set \litsett of classical literals, we say all literals in \litsett are \emph{simultaneously derivable} iff there exists an answer set $\litset \in AS(\pp)$ such that $\litsett \subseteq \litset$.
Atoms that occur in the rule bodies of \pp but not in any rule head will be called \emph{external atoms} and the set of external atoms will be denoted by \extAtoms.
Atoms that occur in \pp that are not external will be called \emph{internal atoms}.
As stated in \citep{GelderRossSchlipf1991}, deductive databases are commonly viewed as logic programs as they consist of an \emph{external database}, which is a set of facts,  and the \emph{internal database}, which is a set of rules.
In software applications, often the same program is used with different input data.
This can also be applied to extended logic programs.
Similar to the division inside databases and the distinction between \emph{problem instance and encoding} as mentioned before, we partition an ELP \pp into a set \inputpp of facts which we will call \emph{input}, and a set \corepp of rules which will be called \emph{program core}.
We define that any \emph{valid input} \inputpp for \corepp is a consistent set of facts over \LitE and that \corepp can only comprise complex rules.
The set of all valid inputs \inputpp for a program core \corepp will be denoted by $I(\corepp)$.
By $\Pi(\corepp) = \{\corepp \cup \inputpp \mid \inputpp \in I(\corepp)\}$, we denote the set of all programs \corepp extended by a valid input \inputpp for \corepp.

In Section~\ref{subsec:conflicts}, we will explain why it is necessary to restrict \inputpp to literals in \LitE.

As mentioned above, default negation puts a classical literal $L$ into a binary state \wrt a set \litset of classical literals.
Either $L$ is true in \litset while $\nottt L$ is false in \litset or vice versa.
With regard to strong negation however, this is not the case.
In Figure~\ref{fig:negation}, this distinction is illustrated by the \emph{gap} between $true$ and $false$ on the level of $A$ and $\nlit A$.
Under three-valued semantics \citep{Przymusinski1991a}, this additional \emph{state} is referred to as \emph{undefined}.
Atoms $A$ and $\nlit A$ are undefined in \litset if neither $A$ nor $\nlit A$ are derivable in \litset.
Under answer set semantics, this undefinedness is covered by default literals, \ie, both $\nottt A$ and $\nottt \nlit A$ are simultaneously true whenever neither $A$ nor $\nlit A$ is derivable.
It is for this reason that we introduced the notion of reconcilable literals.
Even though for two reconcilable literals $\nottt A, \nottt \nlit A$, the atoms $A,\nlit A$ ``inside'' the literals are strongly complementary, the default negation allows that both literals are true simultaneously in $\litset$, \ie, $S \vDash \{\nottt A,\nottt \nlit A\}$ whenever $A, \nlit A \notin S$.

\subsection{Consistency}%
\label{subsec:consistency}
In \citep{Gelfond1991}, an answer set program \pp over $\calA$ is \emph{inconsistent} if \pp has no answer sets or its only answer set is $\LitA$.

  We will specify inconsistency in more detail by using the results in~\citep{Schulz2015} where the authors define four different cases how and why an ELP can be inconsistent.
  There, in the first three cases \emph{1, 2} and \emph{3a}, the inconsistency is essentially caused by classically complementary literals in the program, \ie, by their simultaneous derivation.
  In the fourth case \emph{3b}, a so called \emph{negative dependency path} exists that leads to the derivation of default-complementary literals.
  From a semantic viewpoint, in case \emph{1} and \emph{2}, \pp has no well-founded models.
  In case \emph{3a}, \pp has a well-founded model, and the normal program \ppgen of \pp has one or more answer sets.
  In case \emph{3b}, \pp has a well-founded model but \ppgen has no answer sets.
  We will say that an inconsistent ELP \pp is \emph{contradictory} if the inconsistency is caused by classically complementary literals (case 1, 2 or 3a).
  Otherwise, \pp will be called \emph{incoherent} (case 3b)\footnote{Note that the inconsistency type is not solely determined by the answer set semantics of \pp but also by using well-founded semantics and the normal program of \pp.
  For more information on the inconsistency types, we refer the reader to~\citep{Schulz2015,Inoue1993}.}.
  In this work, we will only consider contradictory programs.
  The handling of incoherent programs is examined in~\citep{ThevapalanHeyninckKernIsberner2021} as well as in~\citep{Costantini2006}.

The presented method will guarantee that a modified ELP \pp is not contradictory with any valid input \inputpp.
For this reason we introduce the notion of \emph{uniformly non-contradictory program cores}.
\begin{definition}[Uniformly Non-Contradictory Program Core]\label{def:consistent-idb}
  A program core \corepp of an ELP \pp over $\calA$ is \emph{uniformly non-contradictory} if for every valid input \inputpp for \corepp, $\corepp \cup \inputpp$ is not contradictory.
\end{definition}
\pagebreak
\begin{corollary}\label{cor:nontradictory}
  Given a uniformly non-contradictory program core \corepp over \pAtoms, for every $\pp \in \Pi(\corepp)$ the following holds:  either $S \neq \LitA$ for every $S \in AS(\pp)$, or \pp is incoherent.
\end{corollary}
Naturally, given a program core \corepp, every program $\pp \in \Pi(\corepp)$ is uniformly non-contradictory whenever \corepp is.

As stated in \citep{Gelfond1991} and implied by Corollary~\ref{cor:nontradictory}, non-contradictory program cores do not guarantee consistent programs since an ELP that is free of contradictions could potentially still be incoherent.
Detecting incoherence and establishing coherence in an ELP, though, is a separate task \citep[see][]{Schulz2015,Schulz2017} and outside of the scope of this paper.
We, therefore, assume that in the following, any given ELP is coherent if not stated otherwise.

\section{Contradictions}%
\label{sec:contradictions}
The aim of this paper is to present a method that analyzes program cores and computes possible modifications such that the modified program core becomes uniformly non-contradictory.
The presented method will, therefore, ensure that given a program core \corepp, there is no contradictory ELP $\pp \in \Pi(\corepp)$.
To that end, we first \emph{identify} what may cause contradictions.
We will then show how to find the causes (\emph{conflict detection}) and remove them (\emph{conflict resolution}).
In this section, we will assume that any given program core \corepp does not contain constraints.
The influence of constraints with respect to the consistency of the program will be examined in Section~\ref{ssec:conflict_resolution_with_constraints}.

\subsection{Conflicts}%
\label{subsec:conflicts}
To find rules that could potentially lead to contradictions, one has to look at rules with strongly complementary head literals.
We will call two rules with strongly complementary head literals \emph{conflicting} if both rule bodies are simultaneously satisfiable by a consistent set of classical literals.

\begin{definition}[Conflicting Rules, Conflict \citep{Thevapalan2020}]\label{def:conflict}
  Suppose an ELP \pp.
  Two rules ${r,r'\in \corepp}$, $r \neq r'$, are \emph{conflicting} (written as $r \conf r'$) if $\head$ and $\head[r']$ are strongly complementary and there exists a consistent set of classical literals ${\litset \subseteq \LitA}$ such that $\body$ and $\body[r']$ are true in $\litset$.
A \emph{conflict} is a pair $(r,r')$ of rules such that $r,r'$ are conflicting.
  We will denote the set of all conflicts ${(r,r')}$ in an ELP \pp by $\Confl{\pp}$, and correspondingly the set of all conflicts $(r,r')$ involving a rule $r$ will be denoted by \confl{r}.
  Furthermore, we will refer to the set of rules $r' \in \pp$ that are conflicting with $r$ in \pp as \emph{adversarial rules} of $r$, denoted by $\conflr{r}$, \ie, ${\conflr{r} = \{r' \mid r \conf r', r'\in \pp\}}$.
  We will call two rules $r,r'$ \emph{non-conflicting} (symbolically $r \nconf r'$) iff $r \conf r'$ does not hold.
\end{definition}

\paragraph{Remark 1:}
(a) If ${r,r' \subseteq \corepp}$ are conflicting rules, then there exists a consistent set \litset of classical literals such that ${\litset \vDash \body,\body[r']}$.
Consequently, there is a set of facts \dbE  (obtained from \litset in a straighforward way) such that ${\corepp \cup \dbE}$ is contradictory.
(b) The identification of two rules as conflicting is done independently of the rest of the program core and especially independently of any set \dbE of currently given facts.

\begin{table}
  \centering
  \caption{Relationship between two rules $r,r'$ with complementary head literals \wrt a literal ${L=A}$}%
  \label{tab:conf-cases}
  {\tablefont\begin{tabular}{@{\extracolsep{\fill}}rllll}
      \topline
        & $L \in \bodyp{+}$ & $\nlit L \in \bodyp{+}$ & $L \in \bodyp{-}$ & $\nlit L \in \bodyp{-}$
        \midline
      $L \in \bodyp[r']{+}$ & \pc & \nc & \nc & \pc  \\
      $\nlit L \in \bodyp[r']{+}$ & \nc  & \pc & \pc & \nc \\
      $L \in \bodyp[r']{-}$ & \nc  & \pc  & \pc & \pc \\
      $\nlit L \in \bodyp[r']{-}$ & \pc  & \nc & \pc & \pc
      \botline
  \end{tabular}}
\end{table}

\begin{proposition}
  Given an ELP \pp over \pAtoms, its program core \corepp is uniformly non-contradictory if \corepp does not contain any conflicts.
\end{proposition}
\begin{proof}
  Assume an ELP \pp whose core \corepp is not uniformly non-contradictory.
  Then, by Definition~\ref{def:consistent-idb}, there exists a valid input \inputpp such that $\pp=\corepp \cup \inputpp$ is contradictory.
  This implies that there exist two complementary literals $L,\nlit L$ that are simultaneously derivable in \pp.
  $L,\nlit L$ are only simultaneously derivable if either
  (a) $\{L.,\nlit L.\} \subseteq \inputpp$, or
  (b) $L. \in \inputpp$ and there exists a rule $r \in \corepp$ with $\head=\nlit L$ such that $\body$ is satisfied in \pp, or
  (c) there exist two rules $r,r'\in \corepp$ with $\head=L,\head[r']=\nlit L$ such that $\body,\body[r']$ are satisfied in \pp.
  Regarding (a): \inputpp cannot contain fact literals $L,\nlit L$ because \inputpp has to be consistent by definition.
  Regarding (b): \inputpp does only contain facts over external atoms.
  Since, by definition, external atoms do not occur in rule heads of \corepp, $L.$ cannot be in \inputpp whenever there is a rule in \corepp with head literal $\nlit L$.
  Therefore, the only way how $L,\nlit L$ could be derived simultaneously is if (c) holds, \ie, if there exist two rules $r,r'$ in \corepp with complementary rule heads that are satisfiable in \pp.
  Consequently, if \corepp is not uniformly non-contradictory, then \corepp contains at least one conflict.
\end{proof}

In nonmonotonic logics, the possible occurence of contradictions is often dealt with by restricting the syntax of knowledge representation languages.
As mentioned before, restricting the set of possible facts in \inputpp to facts over external atoms of the respective program core limits the expressibility in programs.
But this limitation assures that contradictions via a rule in \corepp and a fact in \inputpp cannot arise even if \corepp is conflict-free since those kinds of contradictions are not caused by conflicts between complex rules in \corepp.
Therefore, in the following, \inputpp will only consist of fact literals over \extAtoms.

\subsection{Conflict Detection}%
\label{subsec:conflict_detection}
We will now analyze the properties of (non-)conflicting rules.
For two rules with complementary heads to be conflicting, both their bodies have to be satisfiable by at least one consistent set of literals.
Table~\ref{tab:conf-cases} shows the different combinations in which an atom ${A \in \pAtoms}$ can appear as a classical literal $L$ in an ELP \pp over \pAtoms in the bodies of two rules $r,r'$ with complementary rule heads.
For each case, the table states whether both rule bodies can hold simultaneously and consequently, whether $L$ makes the two rules explicitly non-conflicting, or a conflict might be possible.
It is easy to see that $r,r'$ are non-conflicting whenever the corresponding extended literals in $\body,\body[r']$ that are based on $A$ are strongly or default complementary.
These observations lead to the following conclusion:

\begin{theorem}\label{thm:confprops}
  Let \pp be a program with rules ${r,r'\in \corepp}$.
  Two rules ${r,r'}$ are conflicting if and only if
  \begin{description}
    \item[(CP1)] $\head,\head[r']$ are strongly complementary, and
    \item[(CP2)] ${\bodyp[r_1]{+} \cap \bodyp[r_2]{-} = \emptyset}$ such that $r_1,r_2 \in \{r,r'\}, r_1 \neq r_2$, and
    \item[(CP3)] ${\bodyp{+} \cup \bodyp[r']{+}}$ is consistent.
  \end{description}
\end{theorem}

\begin{proof}
  Suppose an ELP \pp over \pAtoms, and rules ${r,r'}$ in \corepp that satisfy (CP1), (CP2), and (CP3).
  For rules $r,r'$ to be conflicting according to Definition~\ref{def:conflict}, their head literals have to be strongly complementary, \ie, ${L = \head}$ and ${\nlit L = \head[r']}$.
  This is satisfied via (CP1).
  Definition~\ref{def:conflict} also requires that for $r,r'$ to be conflicting, there has to exist a consistent set $\litset \subseteq \LitA$ of classical literals for which the bodies $\body, \body[r']$ are true.
  We show that ${\litset = \bodyp{+} \cup \bodyp[r']{+}}$ is such a set.
  By (CP3), \litset is consistent.
  In order for \litset to satisfy both rule bodies \body, \body[r'] simultaneously, $\litset$ has to have the following properties by definition: (a) $\litset \cap (\bodyp{-} \cup \bodyp[r']{-}) = \emptyset$ and (b) $\bodyp{+} \cup \bodyp[r']{+} \subseteq \litset$.
  (b) is trivially fulfilled.
  Regarding (a), we determine
  \begin{align}
    \litset \cap (\bodyp{-} \cup \bodyp[r']{-}) &= (\bodyp{+} \cup \bodyp[r']{+}) \cap (\bodyp{-} \cup \bodyp[r']{-}) \nonumber\\
                                                &\begin{multlined}[b]= (\bodyp{+} \cap \bodyp{-}) \cup (\bodyp{+} \cap \bodyp[r']{-}) \\
                                                \cup (\bodyp[r']{+} \cap \bodyp{-}) \cup (\bodyp[r']{+} \cap \bodyp[r']{-})\end{multlined}\label{eqn:proof-cp-3}
  \end{align}
  Due to our assumption (see Remark 1), the first and the last intersection in (\ref{eqn:proof-cp-3}) are empty, and due to (CP2), also the middle ones in (\ref{eqn:proof-cp-3}) are empty.
  Altogether, we have $\litset \cap (\bodyp{-} \cup \bodyp[r']{-}) = \emptyset$ which yields (a).
  So, $(r,r')$ is a conflict in \pp.

  Conversely, let \pp be an ELP with conflicting rules $r,r' \in \corepp$.
  (CP1) holds, since by definition, the head literals of $r$ and $r'$ are strongly complementary.
  Furthermore, there exists a consistent set \litset of classical literals such that $\body$ and $\body[r']$ are true in \litset.
  Since \litset cannot satisfy two classical literals $L$ and $\nottt L$ simultaneously, (CP2) holds.
  As \litset is consistent, its subset $\bodyp{+} \cup \bodyp[r']{+}$ is also consistent.
  Therefore, (CP3) holds.
\end{proof}

Given (CP2) and (CP3) in Theorem~\ref{thm:confprops} and Table~\ref{tab:conf-cases}, we can see that $r,r'$ are not conflicting whenever $\body[r]$ contains a literal $L$ such that $\body[r']$ contains a literal that is either strongly or default-complementary to $L$.

\begin{definition}[Conflict-Preventing Literals]\label{def:conflict-preventing-lits}
  Given an ELP \pp and two rules ${r,r' \in \corepp}$ with complementary heads, two extended literals ${\LitExt \in \body}$, ${\LitExtPrime \in \body[r']}$ are \emph{conflict-preventing} if \LitExt and \LitExtPrime are atom-related and complementary.
\end{definition}

This leads us to the following observation:

\begin{proposition}\label{prop:conf-preventing}
  Let $\pp$ be an ELP with two rules $r,r' \in \pp$ with complementary heads.
  The rules $r,r'$ are non-conflicting iff there exist two extended literals $\LitExt \in \body,\LitExtPrime \in \body[r]$ such that \LitExt,\LitExtPrime are conflict-preventing.
\end{proposition}

\begin{proof}
  Two conflict-preventing literals ${\LitExt \in \body}$, ${\LitExtPrime \in \body[r']}$ are either (a) default-complementary or (b) strongly complementary.
  In case of (a), the rules $r,r'$ will not satisfy (CP2) in Definition~\ref{def:conflict}, and in case of (b), (CP3) is not met.
  On the other hand, if two rules  $r,r'$ with complementary head literals are non-conflicting, then by Definition~\ref{def:conflict} either condition (CP2) or (CP3) or both are not met.
  As shown before, (CP2) (\resp (CP3)) can only be not satisfied by $r,r'$ if their bodies contain default-complementary (\resp strongly complementary) literals \LitExt,\LitExtPrime.
  In both cases, \LitExt,\LitExtPrime are conflict-preventing literals.
\end{proof}

\begin{example}\label{ex:contradictions}
  Consider the ELP $\pp_{\mref{ex:contradictions}}$:
  \begin{align*}
    \rowprefix{r_1} & x \when a, \nottt b.\\
    \rowprefix{r_2} & \nlit x \when b, \nottt c. \\
    \rowprefix{r_3} & \nlit x \when \nlit a, \nottt b. \\
    \rowprefix{r_4} & \nlit x \when a, \nlit b. \\
    \rowprefix{r_5} & \nlit x \when c, d.
  \end{align*}
  $\pp_{\mref{ex:contradictions}}$ has two conflicts: $(r_1,r_4)$ and  $(r_1,r_5)$.
  The rules $r_1, r_2$ violate (CP2) and have the conflict-preventing literals $b,\nottt b$.
  The rules $r_1, r_3$ violate (CP3) and have the conflict-preventing literals $a,\nlit a$.
\end{example}

\subsection{Conflict Resolution}%
\label{subsec:conflict_resolution}
Next, we will show how to resolve conflicts in an ELP \pp.
To resolve all conflicts in \pp, one or both rules of each conflict in $\Confl{\pp}$ have to be modified such that \corepp becomes uniformly non-contradictory.
We will call the modification of \pp to $\pp'$ a \emph{conflict resolution step} whenever at least one conflict is resolved and a sequence of conflict resolution steps $\langle \ppi{1}, \ppi{2}, \dots, \ppi{n} \rangle$ will be called \emph{conflict resolution process}.
We present an approach to compute possible solutions to resolve a conflict in \pp such that the following properties hold:
\begin{itemize}
  \item[(P1)]
    The conflict resolution process is \emph{successful}.
    By that, we mean that a conflict resolution process where in each step a computed resolution option is applied will eventually lead to a non-contradictory program core, \ie, a finite sequence of conflict resolution steps $\langle \ppi{1}, \dots, \ppi{n} \rangle$  such that $\Confl{\ppi{n}} = \emptyset$.
  \item[(P2)]
    Each conflict resolution step is \emph{minimally invasive} as it only consists of extending the body of a conflicting rule.
    This means, for every rule $r \in \corepp$ with $\confl{r} \neq \emptyset$, the corresponding modified rule $r' \in \corepp'$ satisfies that $\head[r'] = \head$ and $\body[r'] \supseteq \body[r]$, and for every other rule $r$, it holds that $r'=r$.
  \item[(P3)]
    Each rule that is modified during the resolution process remains \emph{applicable}.
\end{itemize}

Given an ELP \pp over a set of atoms \pAtoms and a rule $r_i \in \corepp$ with a non-empty set $\confl{r_i}$, we want to modify $r_i$ to a rule $r_i'$ such that every conflict in $\confl{r_i}$ is resolved, \ie, ${\confl{r'_i} = \emptyset}$.
For that, our approach analyzes how the extended literals \LitxA of \pp jointly appear in the rules of $\conflr{r_i}$, and computes \emph{conflict-resolving extensions} for $r_i$.

\begin{definition}\label{def:conflict-res-ext}
  Given an ELP \pp over \pAtoms and a rule $r_i \in \pp$, a \emph{\lambdasol-extension $\lambdasolrule{r_i}$ for $r_i$} is a set of extended literals ${\LitExt \in \LitE}$ such that $\Atom{\lambdasolrule{r_i}} \cap \Atom{\body[r_i]} = \emptyset$.
  A rule $r'_i$ of the form
  \begin{align}\label{eqn:kapminusexample}
    \rowprefix{r'_i} \head[r_i] \when \body[r_i], \lambdasolrule{r_i}.
  \end{align}
  will be called a \emph{\lambdasol-extended rule \wrt $\lambdasolrule{r_i}$}.
\end{definition}

To resolve all conflicts $(r_i,r_j) \in \confl{r_i}$ of a rule $r_i \in \pp$, we want to gather those \lambdasol-extensions $\lambdasolrule{r_i}$ for $r_i$ such that the \lambdasol-extended rule $r'_i$ \wrt $\lambdasolrule{r_i}$ and each $r_j$ are not conflicting, \ie, $\confl{r'_i} = \emptyset$.

\begin{definition}
  Given an ELP \pp over \pAtoms and a rule $r_i \in \pp$ with $\conflr{r_i} \neq \emptyset$, a \lambdasol-extension $\lambdasolrule{r_i}$ for $r_i$ is \emph{conflict-resolving} iff $\conflr{r'_i} = \emptyset$ where $r'_i$ is the \lambdasol-extended rule \wrt $\lambdasolrule{r_i}$.
  We say a conflict-resolving \lambdasol-extension for $r_i$ \emph{resolves all conflicts in $\confl{r_i}$ simultaneously}.
\end{definition}

Hence, in order to obtain a uniformly non-contradictory program core \corepp, we resolve all conflicts in \corepp by extending the bodies of particular conflicting rules by a respective conflict-resolving \lambdasol-extension.

We can show that extending a rule in \pp by a \lambdasol-extension does not lead to additional conflicts.
\begin{proposition}\label{prop:extension-cautious}
  Let $\pp_\tau = \pp \backslash \{r_i\}$ be the set of all rules in \pp other than $r_i$.
  Let furthermore $r'_i$ be the \lambdasol-extended rule \wrt a \lambdasol-extension $\lambdasolrule{r_i}$ and  $\pp' = \pp_\tau \cup \{r'_i\}$ the program where $r_i$ is replaced by $r'_i$.
  Then, $\Confl{\pp'} \subseteq \Confl{\pp}$.
\end{proposition}
\begin{proof}
  Suppose a rule $r_j \in \pp$ such that $r'_i \conf r_j$ holds.
  Then by Proposition~\ref{prop:conf-preventing}, $\body[r_j]$ and $\body[r'_i]$ do not contain conflict-preventing literals.
  Since $\body[r_i] \subseteq \body[r'_i]$, this implies that $\body[r_j]$ and $\body[r_i]$ do not contain conflict-preventing literals either and, therefore, $r_j \conf r_i$ holds.
  Consequently, if a \lambdasol-extended rule $r'_i$ \wrt a \lambdasol-extension $\lambdasolrule{r_i}$ is in conflict with another rule $r_j \in \pp$, then $r_i$ and $r_j$ are already conflicting.
\end{proof}
Note that Proposition~\ref{prop:extension-cautious} holds for general \lambdasol-extensions and not exclusively for conflict-resolving ones.

Extending $\body[r_i]$ with a subset-minimal \lambdasol-extension $\lambdasolrule{r_i}$, therefore, constitutes a \emph{cautious change} in \pp.
As a consequence, a uniformly non-contradictory program core $\corepp$ is obtained by applying changes to \corepp that are \emph{justified} in the technical and logical sense.

\subsection{A Na\"ive Approach: Semi-Normal Completion}%
\label{ssec:semi-normal-completion}
Preventing the derivation of contradictions is a well-known problem.
In default logic for example, \emph{normal default theories} in \citep{Reiter1980} provide a solution to this problem by restricting the form of defaults.
In normal default theory, every default has to be of the form $\frac{\alpha:\beta}{\beta}$, meaning that $\beta$ can only be concluded if $\alpha$ is explicitly true and $\beta$ can be assumed to be true.
This guarantees that the default theory has at least one extension.
This idea was adapted for rules of logic programs by Caminada in \citep{Caminada2006,CaminadaSakama2006}.
In \citep{Caminada2006}, \emph{semi-normal defeasible rules} are introduced. These rules have the following form:
\begin{align}
  \rowprefix{r}  L_0 \when L_1, \dots, L_m, \nottt L_{m+1}, \dots, \nottt L_n, \nottt \nlit L_0.
\end{align}

Obviously, if a program core \corepp only consists of semi-normal defeasible rules, \corepp is also uniformly non-contradictory.
The transformation of a program core \corepp to a uniformly non-contradictory core $SN(\corepp)$ using semi-normal defeasible rules can, therefore, be defined in a straightforward way.
The \emph{semi-normal completion} $SN(\corepp)$ of \corepp arises from the extension of the body of each rule $r \in \corepp$ by the literal that is dual to the head literal of $r$.

\begin{definition}\label{def:semi-normal-completion-program}
  Given a rule $\rowprefix{r} L \when \body.$, the \emph{semi-normal completion} $sn(r)$ of $r$ is a rule of the following form:
  \[\rowprefix{sn(r)} L \when \body, \nottt \nlit L.\]
\end{definition}
\begin{definition}\label{def:semi-normal-completion-program}
  Given a program core \corepp, the \emph{semi-normal completion} $SN(\corepp)$ of \corepp is the program core $SN(\corepp) = \{sn(r) \mid r \in \corepp \}$.
\end{definition}

In case there is a conflict in the initial program core \corepp, this conflict is prevented by the additional literals in the rule bodies of the conflicting rules which are dual to the respective head literal.

It is easy to see that when using a semi-normal completed program core \corepp over \pAtoms, the restriction that a valid input \inputpp for \corepp can only consist of literals over \extAtoms becomes void.
Instead, it would allow the input \inputpp for \corepp to be over \pAtoms.
However, this type of \emph{automated conflict-prevention} does not explicitly resolve the underlying problems, \ie, the actual contradictions that are modelled by the rules of the initial \corepp.

\begin{example}\label{ex:sn-wrong-as}
  Consider the following program core:
  \begin{align*}
    \progprefix{SN(\pp_{\calC,\ref{ex:sn-wrong-as}})}
    \rowprefix{r_1} &\nlit{\mathit{allergicToPeanuts}} \when \nottt \mathit{allergicEvent}. \\
    \rowprefix{r_2} &\mathit{allergicToPeanuts} \when \mathit{testedPositivePA}. \\
    \rowprefix{r_3} &\mathit{canEatPeanuts} \when \nlit{\mathit{allergicToPeanuts}}.
  \end{align*}
  $\pp_{\calC,\ref{ex:sn-wrong-as}}$ states that if someone never had an allergic event, they are not allergic to peanuts.
  If someone is tested positive for a peanut allergy, they are indeed allergic to it.
  If someone is explicitly not allergic to peanuts, they are allowed to eat peanuts.
  The semi-normal completion of $\pp_{\ref{ex:sn-wrong-as},\calC}$ yields:
  \begin{align*}
    \progprefix{SN(\pp_{\calC,\ref{ex:sn-wrong-as}})}
    \rowprefix{sn(r_1)} &\nlit{\mathit{allergicToPeanuts}} \when \nottt \mathit{allergicEvent}, \nottt \mathit{allergicToPeanuts}. \\
    \rowprefix{sn(r_2)} &\mathit{allergicToPeanuts} \when \mathit{testedPositivePA}, \nottt \nlit{\mathit{allergicToPeanuts}}. \\
    \rowprefix{sn(r_3)} &\mathit{canEatPeanuts} \when \nlit{\mathit{allergicToPeanuts}}, \nottt \nlit{\mathit{canEatPeanuts}}.
  \end{align*}
  Assume the input $\inputpp = \{\mathit{testedPositivePA}.\}$ for $SN(\pp_{\calC,\ref{ex:sn-wrong-as}})$.
  For $\inputpp \cup SN(\pp_{\calC,\ref{ex:sn-wrong-as}})$, we get the answer sets
  \begin{align*}
    S_1 &= \{\mathit{testedPositivePA},\mathit{allergicToPeanuts}\}\text{, and}\\
    S_2 &= \{\mathit{testedPositivePA},\nlit{\mathit{allergicToPeanuts}}, \mathit{canEatPeanuts}\}.
  \end{align*}
  Clearly, $S_2$ shows that $\pp_{\calC,\ref{ex:sn-wrong-as}}$ contains wrong information or lacks crucial information.
  For example, in this case, $r_1$ only looks at allergic events without considering the possibility that one never consumed a peanut and therefore never had a reaction to begin with.
\end{example}

Semi-normal completions of conflicting rules rather ``bypass'' potential contradictions as the complementary statements are just considered separately instead of at the same time, which, of course, leads to contradictions.
This, in turn, means that given a program core \corepp with conflicting rules, solving programs $\pp \in \Pi(SN(\corepp))$ with the semi-normal completed program core will in fact yield answer sets that are not adequate in the professional sense as at least one of the answer sets will represent unintended conclusions as we have seen in Example~\ref{ex:sn-wrong-as}.
From a functional perspective, such an effect is not desirable whatsoever.
Semi-normal completions can, therefore, be in fact regarded as quickfixes in order to assure consistent answer sets regardless of what the resulting answer sets actually represent.
Example~\ref{ex:sn-wrong-as} illustrates that semi-normal program cores do not suffice if one wants to establish a knowledge base that represents adequate information and their relations in the professional sense.
To incrementally \emph{repair} the knowledge base, one has to analyze conflicting rules and explicitly add, remove, or modify rules in order to get rid of all conflicts.

\subsection{Conflict Resolution With Informative Extensions}%
\label{subsec:conflict_resolution_with_informative_extensions}
In the following, we present an approach to obtain uniformly non-contradictory program cores by computing conflict-resolving \lambdasol-extensions for rules $r_i$.
Since \lambdasol-extensions only comprise literals whose atoms occur in the bodies of the adversaries $r_j \in \conflr{r_i}$ and that do not contain literals that are atom-related to $\head[r_i]$,
the extensions are \emph{informative} in the sense that they utilize the body literals of the rules that are involved in the conflicts of $r_i$ without introducing new atoms or using the head literal.
Semi-normal completions are, in comparison, a purely technical device to prevent contradictions and thus not informative.

According to Theorem~\ref{thm:confprops}, for every rule ${r_j \in \conflr{r_i}}$, $\lambdasolrule{r_i}$ has to contain at least one conflict-preventing literal, \ie, an extended literal $\LitExt$ such that there exists an extended literal ${\LitExtPrime \in \body[r_j]}$ and $\LitExt,\LitExtPrime$ are complementary.
This means that a conflict-resolving \lambdasol-extension for $r_i$ can be obtained by finding a set \blanket of literals that contains at least one body literal of each rule $r_j \in \conflr{r_i}$ and that does not contain any literals that occur in the body of $r_i$.
This set \blanket can then be transformed into a conflict-resolving \lambdasol-extension $\lambdasolrule{r_i}$ by negating each literal in this set.
It is easy to see that in order for the \lambdasol-extended rule $r'_i$ to remain applicable, each computed conflict-resolving \lambdasol-extension and therefore each such \blanket has to meet additional constraints which we will explore next.

In order to compute these sets \blanket, we will first adapt the notion of \emph{hitting sets} \citep{Berge1989} as follows:
Let ${\advbodies{r_i} = \{\body[r_j] \mid r_j \in \conflr{r_i}\}}$ be the collection of all rule bodies in \corepp whose rules are in conflict with ${r_i}$.
A \emph{hitting set} of $\advbodies{r_i}$ is a subset ${h \subseteq \LitxA}$ that meets every set in $\advbodies{r_i}$, \ie, ${\body[r_j] \cap h \neq \emptyset}$ for every ${\body[r_j] \in \advbodies{r_i}}$.
A hitting set $h$ is \emph{minimal} if there does not exist a proper subset of $h$ that is also a hitting set.
It is easy to see that minimal hitting sets of $\advbodies{r_i}$ provide a good starting point to attain suitable conflict-resolving extensions $\lambdasolrule{r_i}$ since every hitting set shares at least one extended literal with the body of every rule ${r_j \in \conflr{r_i}}$.
One could assume that to compute a set of conflict-preventing literals, it suffices to just take any minimal hitting set of $\advbodies{r_i}$ and simply default- or strongly negate each extended literal of that hitting set.
But as mentioned before, not every hitting set can be used to compute conflict-resolving extensions as there are some more conditions that have to be met.
We have to ensure that a \lambdasol-extended rule $r'_i$ of $r_i$ is still applicable.
The applicability is not given if (a) the conflict-resolving extension itself is already inconsistent or (b) the extended rule body $\body[r'_i]$ becomes inconsistent.
Case (a) applies if a hitting set $h$ contains complementary or reconcilable literals $\LitExt,\LitExtPrime$ as every combination of their strongly and default negated complement is also complementary\footnote{Keep in mind that default literals cannot be strongly negated, \eg, $\nottt \nlit L, \nottt K$ are not negated counterparts for two atom-related literals $\nottt L, \nottt \nlit K$.}.
To ensure that (b) does not apply, we only consider those hitting sets that do not comprise literals that are atom-related to literals in $\body[r_i]$.
Hence, hitting sets that contain complementary or reconcilable literals and hitting sets that share common underlying atoms with $\body[r_i]$ will not be used to compute conflict-resolving \lambdasol-extensions.

\paragraph{Remark 2:}
The restriction emerging from (b) is a bit stricter than necessary as rules with atom-related literals in their bodies are still satisfiable if these atom-related literals are compatible.
But in order to avoid unnecessary technicalities, we will omit this special case here.

All these restrictions lead us to the following extension of hitting sets called \emph{blankets} which we will use to compute proper conflict-resolving \lambdasol-extensions:
\begin{definition}[Blanket]\label{def:blanket}
  A \emph{blanket} \blanket for $\conflr{r_i}$ is a non-empty, $\subseteq$-minimal, consistent set of extended literals ${\LitExt \in \LitxA}$ without reconcilable literals such that ${\Atom{\blanket} \cap \Atom{\body[r_i]} = \emptyset}$ and for each ${r_j \in \conflr{r_i}}$, there exists an extended literal ${\LitExt \in \blanket}$ with ${\LitExt \in \body[r_j]}$.
  We denote the set of all blankets for $\conflr{r_i}$ by $\blanketop{\conflr{r_i}}$.
\end{definition}

\begin{example}\label{ex:prop-no-blanket}
  Suppose the following rules:
  \begin{align*}
    \rowprefix{r_1} a &\when b,c,d. \\
    \rowprefix{r_2} \nlit a &\when b, c. \\
    \rowprefix{r_3} \nlit a &\when b, \nottt e. \\
    \rowprefix{r_4} \nlit a &\when b, \nottt \nlit e. \\
    \rowprefix{r_5} \nlit a &\when b, e.
  \end{align*}
  Consider the following ELP:\@ $\pp_{\ref{ex:prop-no-blanket}} = \{r_1,r_2\}$.
  For $\conflr{r_1}$ in $\pp_{\ref{ex:prop-no-blanket}}$, there does not exist a blanket since every body literal of $r_2$ also occurs in the body of $r_1$.

  Now, consider the ELP $\pp'_{\ref{ex:prop-no-blanket}} = \{r_1,r_3,r_4\}$.
  The only body literals of $r_3,r_4$ that do not occur in $r_1$ are $\nottt e,\nottt \nlit e$.
  Since a blanket for $\conflr{r_1}$ in $\pp'_{\ref{ex:prop-no-blanket}}$ must not contain reconcilable literals, there does not exist a blanket for $\conflr{r_1}$.

  Likewise, for the ELP $\pp''_{\ref{ex:prop-no-blanket}} = \{r_1,r_3,r_5\}$, we observe the following:
  The only body literals of $r_3,r_5$ that do not occur in $r_1$ are $\nottt e, e$.
  Since a blanket for $\conflr{r_1}$ in $\pp''_{\ref{ex:prop-no-blanket}}$ must not contain complementary literals, there does not exist a blanket for $\conflr{r_1}$.
\end{example}

\begin{example}\label{ex:run}
  Suppose an ELP $\pp_{\ref{ex:run}}$ with the following rules:
  \begin{align*}
    \rowprefix{r_1} \exEligX &\when \exCondAAdv. \\
    \rowprefix{r_2} \nlit \exEligX &\when \nottt \exHighLCount, \nottt \exPreTreatedN. \\
    \rowprefix{r_3} \nlit \exEligX &\when \nlit \exHighLCount, \nottt \exCiR. \\
    \rowprefix{r_4} \nlit \exEligX &\when \exHighLCount, \nottt \exPreTreatedM. \\
  \end{align*}
  Program $\pp_{\ref{ex:run}}$ describes the following scenario:
  Suppose a therapy $X$ that was developed for patients with condition $A$.
  Currently only patients with an advanced form of condition $A$ are eligible ($r_1$).
  Furthermore, imagine a specific laboratory value $L$ in a patient's body that can only be determined via an invasive test which is why such a test can only be done once every 6 months at most.
  Recently completed long-term studies now indicate that there are some exceptions where $X$ must not be recommended to a patient with advanced $A$.
  The therapy should not be used on a patient if it is unknown whether they currently have a high or low $L$-count and if it is also unknown if they were treated with drug $N$ sometime in the past ($r_2$).
  If the usage of a substance $R$ is not contraindicated and the patient has a low $L$-count, treatment $X$ must also not be recommended ($r_3$).
  Finally, if patient has a high $L$-count, they should not receive treatment $X$ if it cannot be concluded that the patient was treated with drug $M$ in the past ($r_4$).

  Rule $r_1$ is in conflict with every other rule:
  \begin{align*}
    \confl{r_1} &= \{(r_1, r_2),(r_1,r_3),(r_1,r_4)\} \\
    \conflr{r_1} &= \{r_2, r_3, r_4\}
  \end{align*}
  The literals of the rule bodies in $\conflr{r_1}$ do not share any common underlying atoms with the literals in the body of $r_1$.
  The set $\blanketop{\conflr{r_1}}$ of all blankets for $\conflr{r_1}$ consists of the sets $\blanket_1, \blanket_2, \blanket_3, \blanket_4, \blanket_5$ where
  \begin{align*}
    \blanket_1 &= \{\nottt \exHighLCount, \nlit \exHighLCount, \nottt \exPreTreatedM\}, \\
    \blanket_2 &= \{\nottt \exHighLCount, \nottt \exCiR, \nottt \exPreTreatedM\}, \\
    \blanket_3 &= \{\nlit \exHighLCount, \nottt \exPreTreatedN, \nottt \exPreTreatedM\}, \\
    \blanket_4 &= \{\nottt \exPreTreatedN, \nottt \exCiR, \nottt \exPreTreatedM\}, \text{~and} \\
    \blanket_5 &= \{\exHighLCount, \nottt \exPreTreatedN, \nottt \exCiR\}.
  \end{align*}
\end{example}

For the remainder of this paper, we will assume that for any rule $r_i$ with $\conflr{r_i} \neq \emptyset$, there exists at least one blanket for $\conflr{r_i}$.
This particularly implies that
for any given conflict pair $(r_i,r_j)$, $\Atom{\body[r_j]} - \Atom{\body[r_i]} \neq \emptyset$ has to hold.
We will assume that $r_i$ will be chosen such that this condition is satisfied.
If this is not possible in practice, one or both rules have to be modified more individually, \eg, by introducing new atoms to \pp.
Other solutions will be discussed in Section~\ref{subsec:missing_lambdasol_extension}.
Now, we will describe how to compute all possible conflict-resolving \lambdasol-extensions $\lambdasolrule{r_i}$ given a blanket \blanket for $\conflr{r_i}$.

In order to use a blanket ${\blanket \in \blanketop{\conflr{r_i}}}$ to resolve all conflicts $(r_i,r_j)$ in $\confl{r_i}$, the literals in \blanket have to be negated such that for each conflict pair, the addition of the negated literals to the body of $r_i$ leads to the violation of either property (CP2) or (CP3) in Theorem~\ref{thm:confprops}.
Due to the fact that default and strong negation cannot be used interchangeably as default literals cannot be strongly negated, we propose the following \negLxOp-operator, which for an extended literal \LitExt with ${\atom{\LitExt} = A}$, outputs the set of literals that are complementary to \LitExt:
\begin{equation}\label{eqn:neglit}
  \negLx{\LitExt} =
  \begin{cases}
    \{\nottt A, \nlit A\} & \text{if $\LitExt = A$}\\
    \{\nottt\nlit A, A\} & \text{if $\LitExt = \nlit A$}\\
    \{A\} & \text{if $\LitExt = \nottt A$}\\
    \{\nlit A\} & \text{if $\LitExt = \nottt\nlit A$}\\
  \end{cases}
\end{equation}

Therefore, the \negLxOp-operator can be used to compute the proper conflict-preventing literals.
In fact, we want to extend the \negLxOp-operator such that the application of the resulting operator on a blanket for $r_i$ will yield the set of all corresponding conflict-resolving \lambdasol-extensions $\lambdasolrule{r_i}$.
For this purpose, we have to ensure that the negations of all literals in a blanket yield extended rules of the form (\ref{eqn:kapminusexample}) that are still applicable.
Consequently a blanket must not contain complementary literals.
Dual literals however are allowed in a blanket since their negation via the \negLxOp-operator leads to non-complementary literals.

Additionally, since we want minimal sets of conflict-resolving literals, given a blanket with dual literals $\LitExt,\LitExtPrime$, we only want to consider those negated forms that both $\LitExt,\LitExtPrime$ have in common.
Therefore, we extend our \negLxOp-operator such that
the negation of atom-related literals only yields the negated form that both literals have in common, \ie,

\begin{align}\label{eqn:neg-lit-dep}
  \negLxdep{\LitExt}{X} = \cap \{\negLx{K^*} \mid \atom{K^*} = \atom{L^*}, K^* \in X\},
\end{align}
where $X$ is a set of extended literals.
Applied on a blanket and their literals, the resulting literals uphold a crucial property:
\begin{corollary}\label{cor:neg-dep-id-disj}
  Given a blanket \blanket, (a) for all atom-related literals ${\LitExt,\LitExtPrime \in \blanket}$, $\negLxdep{\LitExt}{\blanket} = \negLxdep{\LitExtPrime}{\blanket}$, and (b) for all other literals $\LitExt,\LitExtPrime$, $\negLxdep{\LitExt}{\blanket} \cap \negLxdep{\LitExtPrime}{\blanket} = \emptyset$.
\end{corollary}

\begin{example}\label{ex:neg-common-1}
  Suppose the following sets
  $X_1 = \{a, \nlit b, \nottt b\}$,
  $X_2 = \{a, b, \nlit b\}$, and
  $X_3 = \{a, b, \nottt b\}$.
  The negation of these sets via the \negLxOp-operator yields
  \begin{align*}
    \negLxdep{a}{X_1} = \negLxdep{a}{X_2} = \negLxdep{a}{X_3} &= \{\nlit a, \nottt a\}, \\
    \negLxdep{\nlit b}{X_1} = \negLxdep{\nottt b}{X_1} &= \{b, \nottt\nlit b\} \cap \{b\} = \{b\}, \\
    \negLxdep{b}{X_2} = \negLxdep{\nlit b}{X_2} &= \{\nottt b, \nlit b\} \cap \{\nottt \nlit b, b\} = \emptyset, \text{~and}\\
    \negLxdep{b}{X_3} = \negLxdep{\nottt b}{X_3} &= \{\nlit b, \nottt b \} \cap \{b\} = \emptyset.
  \end{align*}
  The negation of dual literals in $X_1$ leads to the output of the common literal $b$ whereas complementary literals like in $X_2$ and $X_3$ do not share common negated forms.
\end{example}

Example~\ref{ex:neg-common-1} illustrates why blankets are allowed to contain dual literals since they share a common negated form, but must not have complementary literals.
When looking at the negation of $X_2$ in Example~\ref{ex:neg-common-1}, one can see that strongly negated literals still have negated forms that are not complementary, \viz, the reconcilable literals $\nottt b, \nottt \nlit b$.
In contrast to complementary literals, reconcilable literals in a rule body will not inhibit the satisfiability of the rule.
The examples, therefore, demonstrate that by requesting common negated forms, our approach is prima facie stricter than necessary.
But in order to keep this approach simple while guaranteeing that conflict-resolving \lambdasol-extensions maintain the satisfiability of the extended rule, we will only allow combinations of atom-related literals whose negations are guaranteed to be non-complementary.

\begin{lemma}\label{lem:neg-lit-blanket}
  Let an ELP \pp and a rule ${r \in \pp}$ be given.
  For each extended literal \LitExt in a blanket \blanket for \conflr{r}, \negLxdep{\LitExt}{\blanket} is a non-empty set of extended literals which are complementary to \LitExt.
\end{lemma}

\begin{proof}
  Since the \negLxOp-operator as defined in (\ref{eqn:neglit}) returns a non-empty set for any extended literal \LitExt, \negLxdep{\LitExt}{\blanket} can only be an empty set if \blanket contains some distinct atom-related literals $\LitExtPrime_1,\dots,\LitExtPrime_m$ such that $\negLx{\LitExtPrime_1} \cap \dots \cap \negLx{\LitExtPrime_m} = \emptyset$.
  Suppose atom-related literals $\LitExtPrime_1, \LitExtPrime_2, \LitExtPrime_3$.
  By (\ref{eqn:neglit}), $\negLx{\LitExtPrime_1} \cap \negLx{\LitExtPrime_2} \neq \emptyset$ iff $\LitExtPrime_1,\LitExtPrime_2$ are dual.
  In all other cases, $\negLx{\LitExtPrime_1} \cap \negLx{\LitExtPrime_2} = \emptyset$.
  Let therefore $\LitExtPrime_1, \LitExtPrime_2$ be dual literals.
  For any $\LitExtPrime_3$, it is easy to see that by (\ref{eqn:neglit}) either $\negLx{\LitExtPrime_1} \cap \negLx{\LitExtPrime_3} = \emptyset$ or $\negLx{\LitExtPrime_2} \cap \negLx{\LitExtPrime_3} = \emptyset$
  holds.
  This is due to the fact that in every set of at least three different atom-related literals, there are at least two literals that are not dual.
  Consequently, $\negLxdep{\LitExt}{\blanket} = \emptyset$ if \blanket contains more than two atom-related literals or if there exist two atom-related literals $\LitExtPrime_1,\LitExtPrime_2 \in \blanket$ such that $\LitExtPrime_1,\LitExtPrime_2$ are not dual.
  A blanket does not contain complementary or reconcilable literals by definition and only allows atom-related literals that are dual.
  This implies that for each extended literal $\LitExtPrime_1$ in \blanket, there exists at most one other extended literal $\LitExtPrime_2$ in \blanket such that $\LitExtPrime_1,\LitExtPrime_2$ are atom-related, and all atom-related literals in \blanket are dual.
  Thus, for each extended literal ${\LitExt \in \blanket}$, \negLxdep{\LitExt}{\blanket} is not empty.
  More precisely, \negLxdep{\LitExt}{\blanket} always contains either $\nlit L$ or $\nottt L$ or both.
  If \blanket does not contain literals that are atom-related to \LitExt, then $\negLxdep{\LitExt}{\blanket}=\negLx{\LitExt}$ holds.
\end{proof}

For a set $X$ of extended literals \LitExt, we define $\NegLx{X}$ as the set of all possible sets $T$ where $T$ contains every literal of $X$ in a negated form, \ie,
\begin{align}\label{eqn:neg-for-set}
  \begin{aligned}
    \NegLx{X} = \{T \mid T \subseteq \Lxall{X},~&T~\text{contains exactly one element of}\\
                                                 &\text{each non-empty set } \negLxdep{\LitExt}{X}, \LitExt \in X\},
  \end{aligned}
\end{align}
where $\Lxall{X} = \bigcup\limits_{\LitExt \in X} \negLxdep{\LitExt}{X}$.
Due to Lemma~\ref{lem:neg-lit-blanket}, there always exists at least one such set $T$.

\begin{table}
  \centering
  \caption{Results from Example~\ref{ex:negset}}%
  \label{tbl:ex-neg-results}
  {\tablefont\begin{tabular}{p{.04\textwidth}p{.48\textwidth}p{.48\textwidth}}
      \topline
      $\blanket_i$& $\negLxdep{\LitExt}{\blanket_i}$& $\NegLx{\blanket_i}$
      \midline
      $\blanket_1$ & $\{\exHighLCount\},\{\exPreTreatedM\}$ & $\big\{\{\exHighLCount, \exPreTreatedM\}\big\}$ \\
      $\blanket_2$ & $\{\exHighLCount\},\{\exCiR\}, \newline ~~~\{\exPreTreatedM\}$ & $\big\{\{\exHighLCount, \exCiR, \exPreTreatedM\}\big\}$ \\
      $\blanket_3$ & $\{\exHighLCount, \nottt \nlit \exHighLCount\},\{\exPreTreatedN\}, \newline~~~\{\exPreTreatedM\}$ & $\big\{\{\exHighLCount, \exPreTreatedN, \exPreTreatedM\}, \newline~~~ \{\nottt \nlit \exHighLCount, \exPreTreatedN, \exPreTreatedM\}\big\}$ \\
      $\blanket_4$ & $\{\exPreTreatedN\},\{\exCiR\}, \newline~~~ \{\exPreTreatedM\}$ & $\big\{\{\exPreTreatedN, \exCiR, \newline~~~ \exPreTreatedM\}\big\}$ \\
      $\blanket_5$ & $\{\nlit \exHighLCount,\nottt \exHighLCount\},\{\exPreTreatedN\}, \newline~~~ \{\exCiR\}$ & $\big\{\{\nlit \exHighLCount, \exPreTreatedN, \exCiR\}, \newline~~~\{\nottt \exHighLCount, \exPreTreatedN, \exCiR\}\big\}$
      \botline
  \end{tabular}}
\end{table}

\begin{example}[Example~\ref{ex:neg-common-1} continued]
    Continuing Example~\ref{ex:neg-common-1}, we get $\NegLx{X_2}=\NegLx{X_3} = \big\{\{\nlit a\},\{\nottt a\}\big\}$ since $\negLxdep{a}{X_2}=\negLxdep{a}{X_3}=\{\nlit a, \nottt a\}$, and the negation of all other literals yields an empty set.
    For $X_1$, however, we get $\NegLx{X_1}=\big\{\{\nlit a,b\},\{\nottt a,b\}\big\}$ because as for $X_2,X_3$, the negation of $a$ produces two different possible negations and additionally $\nlit b$ and $\nottt b$ have the common negated form $b$.
\end{example}
\begin{example}[Example~\ref{ex:run} continued]%
  \label{ex:negset}
  Continuing Example~\ref{ex:run}, the negation of each blanket in $\blanketop{\conflr{r_1}}$ for rule $r_1$ in $\pp_{\ref{ex:run}}$ yields sets $\blanket_1,\dots,\blanket_5$ as shown in Table~\ref{tbl:ex-neg-results}.
  Blankets $\blanket_2$ and $\blanket_4$ only contain default literals and therefore, each of them has only one possible negation form.
  For $\blanket_3$ and $\blanket_5$, the literals $\exHighLCount$, and $\nlit \exHighLCount$, each have two possible negations which is why $\NegLx{\blanket_3}$ and $\NegLx{\blanket_5}$ contain two sets respectively.
  Regarding $\blanket_1$, we observe that the dual literals $\nlit \exHighLCount,\nottt \exHighLCount$, have the common negated form $\exHighLCount$, and the default literal $\nottt \exPreTreatedM$ has only one negation form.
  Together, that results in the fact that $\NegLx{\blanket_1}$ contains only the one set $\{\exHighLCount,\exPreTreatedM\}$.
\end{example}

\begin{lemma}\label{lem:neg-set-blanket}
  Let an ELP \pp, a rule ${r \in \pp}$ and a blanket \blanket for \conflr{r} be given.
  Furthermore, let $\lambdasol \in \NegLx{\blanket}$.
  For each extended literal ${\LitExt \in \blanket}$, $\lambdasol$ contains exactly one extended literal \LitExtPrime such that $\LitExt,\LitExtPrime$ are complementary.
\end{lemma}
\begin{proof}
  Lemma~\ref{lem:neg-lit-blanket} states that for each ${\LitExt \in \blanket}$, $\negLxdep{\LitExt}{\blanket}$ consists of either one or two extended literals and each of them is complementary to \LitExt.
  Consequently, it follows by (\ref{eqn:neg-for-set}) that each \lambdasol contains a literal \LitExtPrime for each \LitExt in \blanket such that \LitExt,\LitExtPrime are complementary.
  With Corollary~\ref{cor:neg-dep-id-disj}, there exists exactly one such extended literal $\LitExtPrime \in \lambdasol$ for each $\LitExt \in \blanket$.
\end{proof}

Due to the fact that $\blanketop{\conflr{r_i}}$ is a set of sets, we will extend the $\NegLxOp$-operator to also consider sets ${\calX = \{X_1, \dots, X_n\}}$ of sets ${X_1,\dots,X_n}$ of extended literals.
For such a set $\calX$, $\NegLX{\calX}$ shall therefore output the set of all negated form variations in all sets ${X_1,\dots,X_n} \in \calX$, \ie,
\begin{align}\label{eqn:negset}
  \NegLX{\calX} = \{X^- \mid X^- \in \NegLx{X}, X \in \calX\}.
\end{align}
\pagebreak
\begin{lemma}\label{lem:neg-consistent-blanket}
  Let an ELP \pp and a rule ${r \in \pp}$ be given.
  Each ${\lambdasol \in \NegLX{\blanketop{\conflr{r}}}}$ does not contain any complementary literals.
\end{lemma}
\begin{proof}
  Since every \lambdasol consists of a literal $\LitExt \in \negLxdep{\LitExt}{\blanket}$ for every literal $\LitExt \in \blanket$, \lambdasol can only contain complementary literals if there are complementary literals between sets $\negLxdep{\LitExt}{\blanket}$.
  By Corollary~\ref{cor:neg-dep-id-disj}, this is not possible since \blanket itself does not contain complementary or reconcilable literals.
  Consequently, for each ${\lambdasol \in \NegLX{\blanketop{\conflr{r}}}}$, it holds that \lambdasol does not contain complementary literals.
\end{proof}

\begin{example}[Example~\ref{ex:negset} continued]\label{ex:negbag}
  For $\blanketop{\conflr{r_1}}$ in Example~\ref{ex:run}, we get
  \begin{align*}
    \NegLX{\blanketop{\conflr{r_1}}} &= \{\NegLx{\blanket_1}, \NegLx{\blanket_2}, \NegLx{\blanket_3}, \NegLx{\blanket_4}, \NegLx{\blanket_5}\} \\
                                     &= \big\{\{\exHighLCount, \exPreTreatedM\}, \\
                                     &\{\exHighLCount, \exCiR, \exPreTreatedM\}, \\
                                     &\{\exHighLCount, \exPreTreatedN, \exPreTreatedM\}, \\
                                     &\{\nottt \nlit \exHighLCount, \exPreTreatedN, \exPreTreatedM\}, \\
                                     &\{\exPreTreatedN, \exCiR, \exPreTreatedM\}, \\
                                     &\{\nlit \exHighLCount, \exPreTreatedN, \exCiR\}, \\
                                     &\{\nottt \exHighLCount, \exPreTreatedN, \exCiR\}\big\}.
  \end{align*}
\end{example}

Due to the properties of a blanket and their possible negations, it is also guaranteed that any rule extended by such a \lambdasol remains applicable.

\begin{proposition}\label{prop:lambda-extrule-applicable}
  Every $\lambdasol$-extended rule $r'_i$ of the form
  \begin{align}\label{eqn:confresolved-form}
    r'_i \in \{ \head[r_i] \when \body[r_i], \lambdasolrule{r_i}. \mid \lambdasolrule{r_i} \in \NegLX{\blanketop{\conflr{r_i}}}\}
  \end{align}
  of a rule $r_i \in \pp$ is applicable.
\end{proposition}

\begin{proof}
  Suppose a \lambdasol-extended rule $r'_i$ \wrt a \lambdasol-extension $\lambdasolrule{r_i}$.
  Rule $r'_i$ is not applicable if
  (a) $r_i$ is not applicable,
  (b) $\lambdasolrule{r_i}$ is not satisfiable, or
  (c) $\body[r_i] \cup \lambdasolrule{r_i}$ is not satisfiable.
  Case (a) is not possible due to our initial assumption that all rules in a given logic program are applicable.
  Case (b) can only hold if $\lambdasolrule{r_i}$ contains complementary literals.
  By Lemma~\ref{lem:neg-consistent-blanket}, this is not possible for any conflict-resolving extension $\lambdasolrule{r_i}$.
  Then, case (c) can only hold if there exist complementary literals between $\body[r_i]$ and $\lambdasolrule{r_i}$.
  By Definition~\ref{def:blanket}, a blanket for $r_i$ cannot contain literals that are atom-related to any literals in $\body[r_i]$.
  Consequently, any conflict-resolving extension $\lambdasolrule{r_i}$ for $r_i$ does also not contain any literals that are atom-related to a literal in $\body[r_i]$.
  Therefore, any $\lambdasol$-extended rule of the form (\ref{eqn:confresolved-form}) is applicable.
\end{proof}

We are now ready to relate conflict-resolving \lambdasol-extensions to blankets.
\pagebreak
\begin{proposition}\label{prop:confresolved}
  Every $\lambdasolrule{r_i} \in \NegLX{\blanketop{\conflr{r_i}}}$ is a conflict-resolving \lambdasol-extension for $r_i$, \ie, for every rule \lambdasol-extended rule $r'_i$ \wrt $\lambdasolrule{r_i}$, it holds that $\confl{r'_i} = \emptyset$ in $\pp' = \pp \backslash \{r_i\} \cup \{r'_i\}$.
\end{proposition}

\begin{proof}
  Assume, by way of contradiction, that there exists a rule $r_j \in \conflr{r_i}$ such that the \lambdasol-extended rule $r'_i$ \wrt $\lambdasolrule{r_i}$ and $r_j$ are conflicting, \ie, $r'_i \conf r_j$.
  Then by Theorem~\ref{thm:confprops}, (CP2) and (CP3) hold, \ie, $\bodyp[r'_i]{+} \cap \bodyp[r_j]{-} = \bodyp[r_j]{+} \cap \bodyp[r'_i]{-} = \emptyset$, and
  $\bodyp[r'_i]{+} \cup \bodyp[r_j]{-}$ and $\bodyp[r'_i]{-} \cup \bodyp[r_j]{+}$ are consistent.
  Let \blanket be a blanket on which \lambdasolrule{r_i} is based on, \ie, $\blanket \in \blanketop{\conflr{r_i}}$ and $\lambdasolrule{r_i} \in \NegLx{\blanketop{\conflr{r_i}}}$.
  According to (\ref{eqn:negset}), this also means that $\lambdasolrule{r_i} \in \NegLx{\blanket}$ holds.
  By Definition~\ref{def:blanket}, there exists an extended literal $\LitExt \in \blanket$ such that $\LitExt \in \body[r_j]$.
  If $L \in \bodyp[r_j]{+}$, then by Lemma~\ref{lem:neg-set-blanket} either $\nlit L \in \lambdasolrule{r_i}$ or $\nottt L \in \lambdasolrule{r_i}$.
  This consequently means that if $L \in \bodyp[r_j]{+}$, then either $\nlit L \in \bodyp[r'_i]{+}$, which means condition (CP3) cannot hold or $L \in \bodyp[r'_i]{-}$, which in turn means that (CP2) cannot hold.
  On the other hand, if $L \in \bodyp[r_j]{-}$, then by Lemma~\ref{lem:neg-set-blanket} and (\ref{eqn:kapminusexample}) $L \in \bodyp[r'_i]{+}$, and thus (CP2) cannot hold.
  In any case, $\lambdasolrule{r_i}$ adds an extended literal to the body of $r_i$ such that (CP2) or (CP3) cannot hold for rules $r'_i,r_j$.
  Therefore, $r'_i$ and $r_j$ cannot be conflicting.
\end{proof}

\begin{example}[Example~\ref{ex:negbag} continued]\label{ex:run-solution}
  All conflicts in $\Confl{\pp_{\ref{ex:run}}}$ of Example~\ref{ex:run} can be resolved by replacing $r_1$ with a \lambdasol-extended rule
  \begin{align*}
    \rowprefix{r'_1} \exEligX \when \exCondAAdv, \lambdasolrule{r_1}.
  \end{align*}
  \wrt $\lambdasolrule{r_1}$,
  where
  \begin{align*}
    \lambdasolrule{r_1} \in \NegLX{\blanketop{\conflr{r_1}}} = \big\{&\{\exHighLCount, \exPreTreatedM\}, \\
                                                                     &\{\exHighLCount,\exCiR,\exPreTreatedM\}, \\
                                                                     &\{\exHighLCount,\exPreTreatedN,\exPreTreatedM\}, \\
                                                                     &\{\nottt \nlit \exHighLCount, \exPreTreatedN, \exPreTreatedM\}, \\
                                                                     &\{\exPreTreatedN, \exCiR, \exPreTreatedM\}, \\
                                                                     &\{\nlit \exHighLCount, \exPreTreatedN, \exCiR\}, \\
                                                                     &\{\nottt \exHighLCount, \exPreTreatedN, \exCiR\}\big\}.
  \end{align*}
  Every set in $\NegLX{\blanketop{\conflr{r_1}}}$ contains at least one body literal of every rule $r_j \in \conflr{r_1}$ in a negated form and therefore corresponds to a conflict-resolving \lambdasol-extension for $r_1$ in the sense of (\ref{eqn:kapminusexample}).
\end{example}

Looking at the result in Example~\ref{ex:negbag}, it becomes apparent that the $\NegLXOp$-operator for sets of sets does not necessarily output $\subseteq$-minimal sets, \ie, $\NegLX{\blanketop{\conflr{r_1}}}$ contains $\{\exHighLCount,\exPreTreatedM\}$ and two of its supersets $\{\exHighLCount\,\exCiR,\exPreTreatedM\}$ and $\{\exHighLCount,\exPreTreatedN,\exPreTreatedM\}$.
Since $\{\exHighLCount,\exPreTreatedM\}$ already contains literals of every adversarial rule of $r_1$ in a negated form, additional literals become obsolete.
Consequently, any superset of a set in $\NegLX{\blanketop{\conflr{r_1}}}$ can be safely omitted.
Therefore, for an ELP \pp with a rule $r_i$, we denote the set of all $\subseteq$-minimal sets in \NegLX{\blanketop{\conflr{r_i}}} by $\NegLXmin{\blanketop{\conflr{r_i}}}$.

\begin{corollary}\label{cor:negmin-valid-confl-ext}
  All conflicts in $\confl{r_i}$ are resolved simultaneously if $r_i$ is replaced by
  \begin{align}
    r'_i \in \{ \head[r_i] \when \body[r_i], \lambdasolrule{r_i}. \mid \lambdasolrule{r_i} \in \NegLXmin{\blanketop{\conflr{r_i}}}\}.
  \end{align}
\end{corollary}

Every $\lambdasolrule{r_i} \in \NegLXmin{\blanketop{\conflr{r_i}}}$ is, therefore, a \emph{minimal} conflict-resolving \lambdasol-extension for $r_i$.

\begin{example}[Example~\ref{ex:run-solution} continued]\label{ex:run:solution}
  For $r_1$ in $\pp_{\ref{ex:run}}$ of Example~\ref{ex:run}, we get
  \begin{align*}
    \NegLXmin{\blanketop{\conflr{r_1}}} = &\big\{\{\exHighLCount, \exPreTreatedM\}, \\
                                          &\{\nottt \nlit \exHighLCount, \exPreTreatedN, \exPreTreatedM\},\\
                                          &\{\exPreTreatedN, \exCiR, \exPreTreatedM\}, \\
                                          &\{\nlit \exHighLCount, \exPreTreatedN, \exCiR\}, \\
                                          &\{\nottt \exHighLCount, \exPreTreatedN, \exCiR\}\big\}.
  \end{align*}
  Consequently, $\lambdasol_1 = \{\exHighLCount, \exPreTreatedM\}$ is a possible \lambdasol-extension for $r_1$.
\end{example}

  Note that if multiple conflicts are resolved simultaneously where the adversarial rules contain more than one literal, the inclusion of a knowledge expert can be crucial.
  In Example~\ref{ex:run:solution}, any possible \lambdasol-extension for $r_1$ proposes a restriction on the eligibility for therapy $X$ that is slightly stricter than the requirements that are implictly imposed by the adversarial rules.
  That is, the extension by $\lambdasol_1$ leads indeed to the resolution of the conflict. Now however, for instance patients that have a low $L$-count while being allergic to substance $R$ cannot receive the treatment which is not explicitly specified in the original program.
  One reason for these stricter suggestions is that \lambdasol-extension are built to minimally cover the adversarial body literals in order to keep the approach simple and pragmatic.
  It is up to the knowledge expert to decide if the new knowledge that is represented in a \lambdasol-extension is justifiable \wrt to its professional adequacy.
  If the knowledge expert wants to adapt $r_1$ according to all adversarial rules such that $r_1$ ``mirrors'' them, they can consider accepting more than one suggestion which in this case would result in replacing $r_1$ by multiple \lambdasol-extended rules $r'_1$ that guarantee that specific literal combinations are not overlooked, \viz, an additional \lambdasol-extended rule of $r_1$ would have to be added where $\lambdasol = \{\nlit \exHighLCount, \exPreTreatedN, \exCiR\}$ if the knowledge expert wants patients with a low $L$-count to receive treatment $X$ where substance $R$ is contraindicated.

Once all \lambdasol-extensions for a conflict are computed, the corresponding \lambdasol-extended rules are presented to the expert.
They can analyze the suggestions and apply the most suitable solution.
Due to their expertise, the presented suggestions can hint to the underlying cause which can make the decision on the solution straightforward.

Moreover, though it demands for some technical knowledge, the framework could also allow the expert to refine a suggestion and then apply it.
Even though, as shown above, all computed solutions are minimal and guarantee the resolution of the considered conflict, refining a solution based on the expert's knowledge can reinforce the robustness of the solution, \ie, possible  future conflicts after consecutive updates can be prevented.

\begin{example}[Example~\ref{ex:run:solution} continued]\label{ex:run:expert}
    Suppose the knowledge expert has to decide on the solution for the conflict of rule $r_1$.
They are presented with the following suggestions:

\begin{align*}
\rowprefix{r^{(1)}_1} & \exEligX \when \exHighLCount, \exPreTreatedM.\\
 \rowprefix{r^{(2)}_1} & \exEligX \when \nottt \nlit \exHighLCount, \exPreTreatedN, \exPreTreatedM.\\
 \rowprefix{r^{(3)}_1} & \exEligX \when \exPreTreatedN, \exCiR, \exPreTreatedM. \\
 \rowprefix{r^{(4)}_1} & \exEligX \when \nlit \exHighLCount, \exPreTreatedN, \exCiR. \\
 \rowprefix{r^{(5)}_1} & \exEligX \when \nottt \exHighLCount, \exPreTreatedN, \exCiR.
\end{align*}
Presume the following scenario:
Treatment $X$ is applied when the therapy options for advanced condition $A$ that were administered up to now has been exhausted.
Contrary to treatment $M$, which is a more common therapy in this scenario, treatment $N$ is a therapy that is recommended when the standard therapy options have not been effective since $N$ is known to be more complicated and not without risks.
Therefore, the knowledge expert decides that the solution for this conflict should include \exPreTreatedN rather than \exPreTreatedM to reflect that therapy $X$ is only to be recommended if nothing else worked out.
This leaves the expert with rules $r^{(4)}_1$ and $r^{(5)}_1$ as potential candidates.
As therapy $X$ itself is also not without risks, the expert wants to emphasize that any other requirements for this therapy should be defined as strictly as possible, \eg, 
the physicians have to explicitly confirm that specific drugs are not contraindicated.
Accordingly, before suggesting therapy $X$, it should be explicitly known whether the patient has indeed a low L-count.
Therefore, rule $r^{(4)}_1$ as a solution is the most fitting candidate.
\end{example}

Example~\ref{ex:run:expert} highlights the fact that generally the suitability of a solution relies on knowledge that is not necessarily reflected in the program.

  The following example illustrates that even in cases where a custom solution is required, the computed solutions can used as a basis for fitting solutions.

\begin{example}[Example~\ref{ex:run:expert} continued]\label{ex:run:modification}
  Now suppose that at the time of the conflict resolution, studies revealed that in order to prescribe treatment $X$, the patient must not have an allergy towards a substance $S$.
  As $r^{(4)}_1$ was found to be the most suitable suggestion, it suffices to add literal $\nlit \exCiS$ to $\body[r^{(4)}_1]$ in order to require that substance $S$ is not contraindicated.   
  Thus, the expert yields 
\[
  \exEligX \when \nlit \exHighLCount, \exPreTreatedN, \exCiR, \nlit \exCiS.
\]
as the final solution to solve all conflicts of $r_1$ in $\pp_{\ref{ex:run}}$.
This modification consistutes a refinement of a suggested \lambdasol-extension.
\end{example}

  In Sections~\ref{subsec:conflict_order} and~\ref{subsec:lambdasol_scores}, we will show that the problem of knowledge gaps will also pertain to other aspects of conflict resolution.
  We propose to alleviate these shortcomings by implementing suitable interaction mechanisms between the system and the expert.

In the following, we will illustrate a complete conflict resolution step with two more examples.
Example~\ref{ex:complete-1} comprises both short and long conflict-resolving extensions.
Example~\ref{ex:complete-2} shows how redundant extensions are removed via \NegLXminOp.

\begin{example}\label{ex:complete-1}
  Suppose the following ELP $\pp_{\ref{ex:complete-1}}$:

  \begin{align*}
    \rowprefix{r_1} & a \when x, y. \\
    \rowprefix{r_2} & \nlit a \when b, d, \nottt e. \\
    \rowprefix{r_3} & \nlit a \when d, \nottt c, \nottt e. \\
    \rowprefix{r_4} & \nlit a \when b, d, \nottt c.
  \end{align*}

  In $\pp_{\ref{ex:complete-1}}$, rule $r_1$ is in conflict with every other rule:

  \begin{align*}
    \Confl{\pp_{\ref{ex:complete-1}}} = \confl{r_1} &= \{(r_1, r_2),(r_1,r_3),(r_1,r_4)\} \\
    \conflr{r_1} &= \{r_2, r_3, r_4\}
  \end{align*}

  The set $\blanketop{\conflr{r_1}}$ of blankets for $\conflr{r_1}$ consists of the sets $\blanket_1, \blanket_2, \blanket_3, \blanket_4$, with

  \begin{align*}
    \blanket_1 &= \{d\}, \\
    \blanket_2 &= \{b, \nottt c\}, \\
    \blanket_3 &= \{b, \nottt e\}, \\
    \blanket_4 &= \{\nottt c, \nottt e\}.
  \end{align*}

  The negation of each blanket in $\blanketop{\conflr{r_1}}$ results in the following sets:

  \begin{align*}
    \NegLx{\blanket_1} &= \big\{\{\nlit d\}, \{\nottt d\}\big\}\\
    \NegLx{\blanket_2} &= \big\{\{\nlit b, c\}, \{\nottt b, c\}\big\}\\
    \NegLx{\blanket_3} &= \big\{\{\nlit b, e\}, \{\nottt b, e\}\big\}\\
    \NegLx{\blanket_4} &= \big\{\{c, e\}\big\}
  \end{align*}

  Therefore, all conflicts in $\confl{r_1}$ can be resolved by replacing $r_1$ with a rule

  \begin{align*}
    \rowprefix{r'_1} a \when x, y, \lambdasolrule{r_1}.
  \end{align*}

  where

  \begin{align*}
    \lambdasolrule{r_1} \in \NegLX{\blanketop{\conflr{r_1}}} &= \NegLXmin{\blanketop{\conflr{r_1}}} \\
                                                             &= \big\{ \{\nlit d\}, \{\nottt d\}, \{\nlit b, c\}, \{\nottt b, c\}, \{\nlit b, e\}, \{\nottt b, e\},\{c, e\}\big\}.
  \end{align*}

  From a technical point of view, choosing $\lambdasolrule{r_1} \in \{\nlit d, \nottt d\}$ as a conflict-resolving extension for $r_1$ may seem as the most suitable choice prima facie, seeing that the addition of a single literal resolves all three conflicts of $r_1$ simultaneously.
  However, depending on what the rules actually represent, it is possible that a corresponding knowledge expert does not regard $\nlit d$ and $\nottt d$ as viable options for resolving the conflicts of $r_1$, and, instead, picks an extension that contains more than one literal.
\end{example}

\begin{example}\label{ex:complete-2}
  Suppose the following ELP $\pp_{\ref{ex:complete-2}}$:

  \begin{align*}
    \rowprefix{r_1} & x \when a. \\
    \rowprefix{r_2} & \nlit x \when \nlit c, d, \nottt \nlit a, \nottt e. \\
    \rowprefix{r_3} & \nlit x \when a, b, \nlit c. \\
    \rowprefix{r_4} & \nlit x \when d, \nottt \nlit b, \nottt f.
  \end{align*}

  In $\pp_{\ref{ex:complete-2}}$, rule $r_1$ is in conflict with every other rule:

  \begin{align*}
    \Confl{\pp_{\ref{ex:complete-2}}} = \confl{r_1} &= \{(r_1, r_2),(r_1,r_3),(r_1,r_4)\} \\
    \conflr{r_1} &= \{r_2, r_3, r_4\}
  \end{align*}

  The set $\blanketop{\conflr{r_1}}$ of blankets for $\conflr{r_1}$ consists of the sets
  \begin{align*}
    \blanket_1 &= \{\nlit c, d\}, \\
    \blanket_2 &= \{\nlit c, \nottt \nlit b\}, \\
    \blanket_3 &= \{\nlit c, \nottt f\}, \\
    \blanket_4 &= \{b, d\}, \\
    \blanket_5 &= \{b, \nottt \nlit b, \nottt e\}, \text{~and}\\
    \blanket_6 &= \{b, \nottt e, \nottt f\}.
  \end{align*}

  Note that literal $\nottt \nlit a$ does not appear in any blankets of $r_1$ by reason of $\atom{\nottt \nlit a} \in \Atom{\body[r_1]}$.

  The negation of each blanket in $\blanketop{\conflr{r_1}}$ yields the following sets:
  \begin{align*}
    \NegLx{\blanket_1} &= \big\{\{c, \nlit d\}, \{\nottt \nlit c, \nlit d\}, \{c, \nottt d\}, \{\nottt \nlit c, \nottt d\}\big\}\\
    \NegLx{\blanket_2} &= \big\{\{c, \nlit b\}, \{\nottt \nlit c, \nlit b\}\big\}\\
    \NegLx{\blanket_3} &= \big\{\{c, f\}, \{\nottt \nlit c, f\}\big\}\\
    \NegLx{\blanket_4} &= \big\{\{\nlit b, \nlit d\}, \{\nottt b, \nlit d\}, \{\nlit b, \nottt d\}, \{\nottt b, \nottt d\}\big\}\\
    \NegLx{\blanket_5} &= \big\{\{\nlit b, e\}\big\}\\
    \NegLx{\blanket_6} &= \big\{\{\nlit b, e, f\}, \{\nottt b, e, f\}\big\}\\
  \end{align*}

  This results in
  \begin{align*}
    \NegLX{\blanketop{\conflr{r_1}}} = \bigcup\limits_{\blanket \in \blanketop{\conflr{r_1}}} \NegLx{\blanket},
  \end{align*}
  and on the grounds that $\{\nlit b,e,f\} \supset \{\nlit b,e\}$, we get
  \begin{align*}
    \NegLXmin{\blanketop{\conflr{r_1}}} = \NegLX{\blanketop{\conflr{r_1}}} \backslash \{\nlit b, e, f\}.
  \end{align*}
  Therefore, all conflicts in $\confl{r_1}$ can be resolved by replacing $r_1$ with a rule
  \begin{align*}
    \rowprefix{r'_1} x \when a, \lambdasolrule{r_1}.
  \end{align*}
  where
  \begin{align*}
    \lambdasolrule{r_1} \in \NegLXmin{\blanketop{\conflr{r_1}}} = \big\{
        &\{c, \nlit d\}, \{\nottt \nlit c, \nlit d\}, \{c, \nottt d\}, \{\nottt \nlit c, \nottt d\}
        \{c, \nlit b\}, \{\nottt \nlit c, \nlit b\},\\
        &\{c, f\}, \{\nottt \nlit c, f\},
        \{\nlit b, \nlit d\}, \{\nottt b, \nlit d\}, \{\nlit b, \nottt d\}, \{\nottt b, \nottt d\}, \\
        &\{\nlit b, e\},
        \{\nottt b, e, f\}
      \big\}.
  \end{align*}
\end{example}

In summary, the results show that implementing a conflict resolution step via conflict-resolving \lambdasol-extensions and blankets yields a conflict resolution process that possesses the properties (P1)-(P3) that were postulated earlier.
By Proposition~\ref{prop:confresolved}, given a rule $r_i \in \pp$, every $\lambdasolrule{r_i} \in \NegLX{\blanketop{\conflr{r_i}}}$ is a conflict-resolving \lambdasol-extension for $r_i$.
Since by every such conflict resolution step, the number of conflicts is reduced, for the final program $\ppi{n}$ of the conflict resolution process, it holds that $\Confl{\ppi{n}} = \emptyset$.
Moreover, we have shown that for the proposed conflict resolution process $\langle \ppi{1},\ppi{2},\dots,\ppi{n}\rangle$ it even holds that $\Confl{\ppi{i+1}} \subsetneq \Confl{\ppi{i}}$ for $1 \leq i < n$.
Thus, (P1) holds.
As every conflict resolution step consists of solely replacing a conflicting rule with its \lambdasol-extension of the form (\ref{eqn:confresolved-form}), (P2) holds.
By Proposition~\ref{prop:lambda-extrule-applicable}, property (P3) holds also.
Therefore, a conflict resolution process using conflict-resolving \lambdasol-extensions leads to a uniformly non-contradictory program core in a minimally invasive manner.
Including a domain-specific expert into the selection of the most suitable \lambdasol-extensions can furthermore ensure that the resulting non-contradictory program core maintains its professionally adequate knowledge base.

\subsection{Missing \lambdasol-Extension}%
\label{subsec:missing_lambdasol_extension}
Under some circumstances, no \lambdasol-extensions can be generated for a rule $r \in \pp$.
In these cases a blanket for \conflr{r} cannot be found.
Determining whether a blanket for \conflr{r} exists can become quite cumbersome and complex since the relationships of atom-related literals between all adversarial rules have to be examined.
A small fraction of such possible literal interactions that can make the finding of a blanket impossible will be illustrated in the following example.
\begin{example}\label{ex:non-ex-lambda}
  Suppose the following program $\pp_{\ref{ex:non-ex-lambda}}$:
  \begin{align*}
    \rowprefix{r_1} & x \when a, d. &
    \rowprefix{r_2} & \nlit x \when a, b. &
    \rowprefix{r_3} & \nlit x \when \nlit c. &
    \rowprefix{r_4} & \nlit x \when c, \nottt b.
  \end{align*}
  Program $\pp_{\ref{ex:non-ex-lambda}}$ has conflicts $(r_1,r_2), (r_1,r_3), (r_1,r_4)$ and there does not exist a blanket for $r_1$.
  This can be easily shown by looking at each potential candidate literal:
  Literal $a$ is already in $\body[r_1]$ and can therefore not be in a blanket.
  Since atom $b$ (\resp $c$) appears in $\body[\Confl{r_1}]$ as default (\resp classically) complementary literals, \viz $r_2,r_4$ (\resp $r_3,r_4$), $b$ (\resp $c$) cannot be part of a blanket either.
\end{example}
There are several workarounds that allow the resolution of such conflicts nevertheless.
One possible way to resolve such conflicts can be the \emph{partitioning} of conflicts, \ie, given a rule $r$ with several adversaries $\conflr{r}$, one could first determine a subset $Adv' \subset \conflr{r}$ for which a blanket can be computed.
This procedure can be repeated until all conflicts of $r'$ are resolved.
The following example illustrates how partitioning can be used to obtain a solution if a specific rule of the program (in this case $r_1$) should primarily be modified.
\begin{example}[Example~\ref{ex:non-ex-lambda} contd.]\label{ex:non-ex-lambda:solve:part}
  Suppose program $\pp_{\ref{ex:non-ex-lambda}}$ from Example~\ref{ex:non-ex-lambda}.
  First, conflicts $(r_1,r_2)$ and $(r_1,r_3)$ can be resolved simultaneously by generating \lambdasol-extensions from the blankets for $Adv' = \{r_2,r_3\}$.
  Two of the resulting extensions are $\lambdasol_1 = \{\nottt b, \nottt \nlit c\}$ and $\lambdasol_2 = \{\nottt b, c\}$.
  Applying $\lambdasol_1$ to $\pp_{\ref{ex:non-ex-lambda}}$ results in program $\pp_{\ref{ex:non-ex-lambda}}'$:
  \begin{align*}
    \rowprefix{r_1'} & x \when a, d, \nottt b, \nottt \nlit c. &
    \rowprefix{r_2} & \nlit x \when a, b. &
    \rowprefix{r_3} & \nlit x \when \nlit c. &
    \rowprefix{r_4} & \nlit x \when c, \nottt b.
  \end{align*}
  Then, the remaining conflict $(r'_1,r_4)$ can be resolved by applying the unique solution $\lambdasol_3 = \{\nottt c\}$ which in turn results in program $\pp_{\ref{ex:non-ex-lambda}}''$:
  \begin{align*}
    \rowprefix{r_1''} & x \when a, d, \nottt b, \nottt \nlit c, \nottt c. &
    \rowprefix{r_2} & \nlit x \when a, b. &
    \rowprefix{r_3} & \nlit x \when \nlit c. &
    \rowprefix{r_4} & \nlit x \when c, \nottt b.
    \end{align*}
  \end{example}

  Note that the reason that a solution like $\lambdasol = \{\nottt b, \nottt c, \nottt \nlit c\}$ cannot be obtained with the general conflict resolution approach is because of the restriction that blankets cannot contain multiple literals of the same atom even though reconcilable literals do not interfer with the applicability of a rule (as can be seen in Example~\ref{ex:non-ex-lambda}).
  But in order to keep the resolution approach simple and pragmatic, such a restriction on blankets was deemed necessary.

  However, whenever conflicts are separately resolved, the final outcome can very much depend on the actual partitions, the order in which the different conflicts are resolved, and the choice of the respective \lambdasol-extensions along the way.
  \begin{example}[Example~\ref{ex:non-ex-lambda:solve:part} contd.]\label{ex:non-ex-lambda:order}
    Reconsider the results in Example~\ref{ex:non-ex-lambda:solve:part}.
    Note that $(r'_1,r_4)$ in $\pp_{\ref{ex:non-ex-lambda}}'$ only has the unique solution $\lambdasol_3 = \{\nottt c\}$ because for the resolution of the previous conflicts $\lambdasol_1$ was chosen.
    Applying $\lambdasol_2$ to $\pp_{\ref{ex:non-ex-lambda}}$ of Example~\ref{ex:non-ex-lambda} instead of $\lambdasol_1$ would result in program $\pp_{\ref{ex:non-ex-lambda}}'''$:
    \begin{align*}
      \rowprefix{r_1'} & x \when a, d, c, \nottt b. &
      \rowprefix{r_2} & \nlit x \when a, b. &
      \rowprefix{r_3} & \nlit x \when \nlit c. &
      \rowprefix{r_4} & \nlit x \when c, \nottt b.
    \end{align*}
    It is easy to see that now for conflict $(r_1,r_4)$, there exist no possible $\lambdasol$-extensions anymore as $\body[r_4] \subset \body[r_1]$.
  \end{example}
  
  Another possible workaround could be a preceding resolution of conflicts of some adversarial rules $r' \in \conflr{r}$, \ie, if a blanket for \conflr{r} cannot be found, one could resolve the conflicts of some $r' \in \conflr{r}$ first and then resolve the remaining conflicts of $r$.
  This workaround is illustrated in the following example:
  \begin{example}[Example~\ref{ex:non-ex-lambda} contd.]\label{ex:non-ex-lambda:solve:flip}
    Again, suppose program $\pp_{\ref{ex:non-ex-lambda}}$ from Example~\ref{ex:non-ex-lambda}.
    Since $r_1$ contains literal $d$ that does not appear in $\body[\Confl{r_1}]$, instead of resolving the conflicts of $r_1$, the conflicts of the adversarial rules could be resolved consecutively, in this case even with the same \lambdasol-extension, namely conflicts $(r_2,r_1), (r_3,r_1), (r_4,r_1)$ with either \lambdasol-extension $\lambdasol_3 = \{\nottt d\}$ or $\lambdasol_4 = \{\nlit d\}$.
    Therefore, one possible resolved program $\pp_{\ref{ex:non-ex-lambda}}''''$ would consist of the following rules:
    \begin{align*}
      \rowprefix{r_1} & x \when a, d. &
      \rowprefix{r_2'} & \nlit x \when a, b, \nottt d &
      \rowprefix{r_3'} & \nlit x \when \nlit c, \nottt d. &
      \rowprefix{r_4'} & \nlit x \when c, \nottt b, \nottt d.
    \end{align*}
  \end{example}

  Examples~\ref{ex:non-ex-lambda:solve:part} and~\ref{ex:non-ex-lambda:solve:flip} show that if the resolution of conflicts of a rule cannot be carried out in one go, there are possible workarounds.
  But, as shown with Example~\ref{ex:non-ex-lambda:order}, even if the conflict resolution is partitioned into multiple ones, the choice of a solution can influence the success of remaining conflict resolution attempts.
  Thus, in such individual cases, both, the order in which the conflicts are solved and the choice of the solutions have to be executed with care.
  Apart from partitioning conflicts and modifying adversarial rules, there is always the possibility to introduce new atoms to the bodies of a conflict, \eg, adding a new atom $A$ to \body and $\nlit A$ to the bodies $\body[r']$ of all adversaries.
  This, however, would violate our requirement of extensions being informative.

    Improving our strategies with respect to such deficiencies right from the beginning, \eg, finding a most suitable order in which conflicts are resolved, is part of our ongoing work.
    Some enhancements of our approach are considered in the next section.
\section{Enhancements}%
\label{sec:extensions}
In the following, we will outline possible enhancements to the presented resolution approach.
Section~\ref{subsec:m_to_n_conflicts} outlines how the multiple resolution processes can be used to resolve many-to-many conflicts and in Section~\ref{ssec:conflict_resolution_with_constraints}, we show how the current resolution approach can also be used to deal with inconsistency-causing constraints.
  In the last two subsections, we will show methods to compute additional information for the knowledge expert to find the most suitable solutions efficiently.
  In Section~\ref{subsec:conflict_order}, we will present a method to sort the detected conflicts in a program by the impact the conflicting rules have on other rules.
  We extend this method in Section~\ref{subsec:lambdasol_scores} to introduce the concept of \emph{\lambdasol-scores} that enable the ordering of \lambdasol-extensions.
\subsection{Many-To-Many Conflicts}%
\label{subsec:m_to_n_conflicts}
  So far, the presented approach handles \emph{1-to-many} conflicts, \ie, the presented approach modifies a single rule $r$ such that the conflicts with $n$ different rules $r' \in \conflr{r}$ are resolved simultaneously.
  It is conceivable that instead of a single rule $r$, a program \pp can contain a set $\calR$ of $m$ different rules $r$ with $\head = L$ ($m > 1$) which results in considering what we will call \emph{many-to-many} conflicts.
  Currently, in those cases the resolution process have to be applied on each such $r \in \calR$ separately and consecutively.
  Depending on the individual case, it can also be reasonable to modify a rule $r' \in \conflr{r}$.
  \begin{example}\label{ex:n-to-m-conflicts}
    Consider the following ELP:\@
    \begin{align*}
      \progprefix{\pp_{\ref{ex:n-to-m-conflicts}}}
      \rowprefix{r_1} x \when a, e. \qquad
      \rowprefix{r_2} x \when b, e. \qquad
      \rowprefix{r_3} \nlit x \when f. \qquad
      \rowprefix{r_4} \nlit x \when g. \qquad
    \end{align*}
    In $\pp_{\ref{ex:n-to-m-conflicts}}$, there are conflicts between $\{r_1,r_2\}$ and $\{r_3,r_4\}$.
    One way to resolve all conflicts is to extend $r_1$ and $r_2$ by a \lambdasol-extension, respectively.
    Since the bodies of $r_3$ and $r_4$ do not contain atom-related literals, each \lambdasol-extension will contain two literals.
    If, on the other hand, the knowledge expert decides to modify $r_3$ and $r_4$ instead, a possible \lambdasol-extension for both rules can consist of a single literal, \viz either $\nottt e$ or $\nlit e$.
  \end{example}

  Example~\ref{ex:n-to-m-conflicts} shows that the order in a conflict pair $(r,r')$ can decide over the ``complexity'' of its resolution.
\subsection{Conflict Resolution with Constraints}%
\label{ssec:conflict_resolution_with_constraints}
In general, (integrity) constraints in answer set programs serve to weed out unwanted answer sets.
One way to implement such a functionality is to map a constraint like $\when a,b.$ to a rule $z \when a,b, \nottt z.$ where $z$ is a newly introduced atom~\citep{GebserKaminskiKaufmannSchaub2012}.
Even though the derivation of $z$ due to such constraints technically causes incoherence and not contradiction, we will cover the handling of a specific type of constraints $r_c$ as we can show that the presented conflict resolution approach can be used to prevent the advent of incoherence due to such rules $r_c$.
In the following, we will examine constraints of the form
\begin{align}\label{eqn:constraint-form}
  \rowprefix{r_c} \when K,L.
\end{align}
where $K,L$ are internal atoms.

Recall that in logic programs, classical negation is actually syntactic sugar:
Given the normal logic program \ppgen of \pp, classical negation can be incorporated by adding a constraint $\when A,A'.$ for every atom $A \in \pAtoms$~\citep{Gelfond1991}.
\begin{example}\label{ex:impl-class-neg}
  Suppose the following program $\pp_{\ref{ex:impl-class-neg}}$:
  \begin{align*}
    \rowprefix{r_1} & x \when a. &
    \rowprefix{r_2} & \nlit x \when b.
  \end{align*}
  This program is actually interpreted as a positive program $\pp_{\ref{ex:impl-class-neg}}^+$:
  \begin{align*}
    \rowprefix{r_1} & x \when a. &
    \rowprefix{r_2'} & x' \when b. &
    \rowprefix{r_x} & \when x, x'.
  \end{align*}
\end{example}
Example~\ref{ex:impl-class-neg} shows that resolving the conflict between $(r_1,r_2')$ also ensures that $r_x$ is satisfied.
As a generalization of this idea, we can therefore consider constraints as implicit pointers to potential conflicts.
Consequently, the satisfaction of a constraint of form (\ref{eqn:constraint-form}) can be ensured by treating every pair of rules $r,r'$ in \pp with $\head = L, \head[r'] = K$ as \emph{implicit conflicts $(r,r')$} and resolving them using the method for the resolution of conflicts via \lambdasol-extensions.

\begin{example}\label{ex:cons:case0:0}
  Suppose the following program $\pp_{\ref{ex:cons:case0:0}}$ with internal atoms $a,d$ and external atoms $b,c,e,f$:
  \begin{align*}
    \rowprefix{r_1} & a \when b, \nottt c. &
    \rowprefix{r_2} & a \when b, \nottt f. &
    \rowprefix{r_3} & d \when e. &
    \rowprefix{r_4} \when a, d.
  \end{align*}
  Then, $(r_3,r_1)$ and $(r_3,r_2)$ can be interpreted as conflicts.
  To resolve these two conflicts simultaneously, $\{\nottt b\}$ can be added as a \lambdasol-extension to $r_3$.
  The modified program $\pp_{\ref{ex:cons:case0:0}}'$ then has the following rules:
    \begin{align*}
      \rowprefix{r_1} & a \when b, \nottt c. &
      \rowprefix{r_2} & a \when b, \nottt f. &
      \rowprefix{r_3'} & d \when e, \nottt b. &
      \rowprefix{r_4} & \when a, d.
    \end{align*}
  \end{example}
  It is easy to see that a constraint of form (\ref{eqn:constraint-form}) can be omitted once the implicit conflicts of the constraint are resolved with fitting \lambdasol-extensions.
  Thus, rule $r_4$ is redundant in program $\pp_{\ref{ex:cons:case0:0}}'$, but one has to keep in mind that future, more comprehensive updates can lead to answer sets that contain literals $a$ and $d$.
\subsection{Conflict Order}%
\label{subsec:conflict_order}
  As stated at the beginning, the presented approach is a basic method that can be extended in many respects depending on the requirements of the user.
  A useful functionality regarding the workflow of conflict resolution is the order in which the conflicts are presented to the expert.
  It is important that the knowledge expert understands how the order of the conflicts is generated as any enhancement in the proposed framework must not negatively influence the actual decision regarding a resolution but rather aid the expert in finding the most suitable solution efficiently.
  One easily comprehensible implementation utilizes the \emph{affected rule count} (inspired by~\citep{AbdelhalimTraoreSayed2009}) of the conflicting literals that reflects how many rules are potentially affected by a specific rule.
  For that, we first define the set $\rulset$ of all rules that are \emph{affected by a rule $r_1$} and compute the rule impact of $r_1$ by counting its affected rules. 
  \begin{definition}[Affected Rules and Affected Rule Count]
    The set  $\rulset_{r_1}$ of \emph{rules affected by a rule $r_1 \in \pp$} is defined as 
    \begin{align*}
    &\rulset_{r_1} = \{r \in \pp \mid \atom{\head[r_1]} \in \Atom{\body}, 
    (\body[r_1] \cup \body) \text{~is consistent}\}.
    \end{align*}
    The \emph{affected rule count} $\ar{r_1}$ of $r_1$ is then defined by the number of affected rules, \viz
  \begin{align*}
    \ar{r_1} = \vert \rulset_{r_1} \vert.
  \end{align*}
  \end{definition}

  \begin{example}[Example~\ref{ex:run} continued]\label{ex:eff-1}
    Let $\pp_{\ref{ex:eff-1}}$ be program $\pp_{\ref{ex:run}}$ extended by the following rules:
    \begin{align*}
      \rowprefix{r_5} \exDrugB &\when \exEligX, \exCondAAdv, \nlit \exDrugE. \\
      \rowprefix{r_6} \exDrugC &\when \exEligX, \nottt \exCondAAdv, \exHighLCount. \\
      \rowprefix{r_7} \exDrugD &\when \nlit \exEligX, \exCondAAdv, \nlit\exPreTreatedN.
    \end{align*}
    In $\pp_{\ref{ex:eff-1}}$, rules $r_5$ and $r_7$ are affected by $r_1$.
    Rule $r_6$ is not affected as $\body[r_1]$ and $\body[r_6]$ are not consistent due to atom $\exCondAAdv$, \ie, literal $\exCondAAdv$ occurs in $\bodyp[r_1]{+}$ and in $\bodyp[r_6]{-}$.
    Ergo for $r_1$, we yield the affected rule count $\ar{r_1} = 2$. 
  \end{example}
  
  The affected rule count can provide the expert with a first impression (or possibly confirm own conjectures)  of the possible impact that the given conflict and consequently the changes made to one of the conflicting rules have within the program.
  It can contribute additional information to the expert to anticipate the effect of a possible solution.
  Hence when using this order-criterion, it is on the expert to decide which conflict they want to solve first, \eg, they can handle those conflicts with the least impact on other rules first.
  Additionally, such computed values can also help to decide which rules of a conflict to modify, \eg, given a rule $r\in \pp$, it can be beneficial to modify one rule in $\confl{r}$ first and then modify $r$ in order to resolve all remaining conflicts of $r$ (see Section~\ref{subsec:m_to_n_conflicts}).
\subsection{\lambdasol-Scores}%
\label{subsec:lambdasol_scores}
In the previous section, we presented a method that aims to improve the user's first major task during the resolution process which is choosing a conflict out of all detected conflicts.
In the next step, the user has to decide on a solution for a picked conflict.
To facilitate the finding of a most suitable solution for a chosen conflict $(r,r')$, an ordering of the computed suggestions can be established.
For that one can assign a score to each \lambdasol-extension that is calculated using different criteria.
In this section, we will outline one prototypical enhancement that adds a numerical value to each generated \lambdasol-extension named \emph{\lambdasol-score}.
  We will present a basic approach to compute \lambdasol-scores by utilizing the previously introduced affected rule count to quantify the \emph{effect} a rule modification can have. 
Suppose a conflict $(r_1,r_2)$ and a solution $r'_1$ which can either be a computed solution (\lambdasol-extended rule) or a modified respectively custom solution for this conflict. 
In order to reflect the effect this solution can have on the overall program, one can look at the size of the difference between the affected rules of $r_1$ and the \lambdasol-extended rule $r_1'$ which we will call the \emph{effect of the modification of $r_1$ to $r'_1$}.
\begin{definition}[Effect]
  Let $(r_1,r_2)$ be a conflict in \pp and $r'_1$ a generated, \lambdasol-extended rule of $r_1$. 
Let furthermore $\rulset_{r_1}$ and $\rulset_{r'_1}$ be the respective sets of affected rules. 
Then, the \emph{effect} $\eff{r_1}{r'_1}$ of this modification is defined by
\begin{align*}
  \eff{r_1}{r'_1} = \vert \rulset_{r_1} - \rulset_{r'_1} \vert.  
\end{align*}
\end{definition}

Note that as \lambdasol-extended rules always expand the body of the original rule and therefore are more specific, any generated solution will affect the same rules or a subset of the rules that are affected by the original rule, \ie, $\rulset_{r'_1} \subseteq \rulset_{r_1}$ holds.
The higher the effect of a solution, the more rules are not affected anymore after the modification \wrt the rules affected by the original rule.
The effect of a suggested modification can therefore be used as the \lambdasol-score of the suggestion.
If for example an expert wants to examine the most cautious solutions, the solutions with the least effect can be presented first.
\begin{example}[Example~\ref{ex:run:solution} continued]\label{ex:eff-2}
  Suppose, the expert wants to resolve the conflict in $\pp_{\ref{ex:eff-1}}$ by changing $r_1$ to $r'_1$ with 
  \[\rowprefix{r'_1} \exEligX \when \nlit \exHighLCount, \exPreTreatedN, \exCiR..\]
  As rule $r_5$ is the only rule affected by $r'_1$, we get the effect 
  \[\eff{r_1}{r'_1} = \vert \{r_5,r_7\} - \{r_5\}\vert = 1.\] 
\end{example} 

The effect of a modification tells the expert how impactful a change is regarding the affected rules, \ie, how many additional and less rules are affected after the changes.
The interpretation of the effect of modifications that is offered at this point can be seen as a first proposal to produce the \lambdasol-scores.
The corresponding method for their computation can thus serve as groundwork where individual enhancements and modifications are possible.%
The presented definition of modification effects is an illustrative example to show how \lambdasol-extensions can hypothetically be ordered.
For a given conflict, the user can inspect the suggestions sorted by their effect, starting with those that have the least effect. 
This way, they can look for and prioritize solutions that have the least impact on the program.

In analogue to the previously described method for ordering conflicts, any implementation for the computation of \lambdasol-scores must not negatively influence the expert regarding their decisions, \ie, enhancements like ordering conflicts and suggestions should not incentivize the expert to choose solutions with the sole intention to resolve all conflicts in the technically easiest and quickest way possible without considering the underlying errors and problems that led to the conflict, nor should it bias the expert's thought process in any other negative way.
  In other words, the goal of such enhancements is to improve the finding of solutions that reflect the actual professional expertise best and simultaneously provide robustness for future updates (\ie mitigate the potential for future conflicts), while not compromising any decision on a solution during the process.

  Establishing a score for each suggestion allows to generate an order over all suggestions which in turn can be used to show the knowledge expert those suggestions first that are more relevant according to their criteria.
  If a \emph{semi-automated} approach is wanted, the knowledge expert could also define a specific score threshold so that suggestions can be applied without the expert's involvement.
  Then, if there exist suggestions \lambdasol with a score below such a threshold, the conflict can be resolved directly by using every \lambdasol with the lowest score as a solution.
  Enhancements like these that gather and use additional \emph{metadata} can help to add a highly customizable abstraction layer between the knowledge expert and the actual knowledge base where the expert can input information about the program elements in order to resolve a conflict without the need for any technical knowledge about answer set programming while ensuring that the resulting rules are still representing correct knowledge in the professional sense.

\section{Conclusion and Future Work}%
\label{sec:conclusion_and_future_work}
In this paper, we proposed a method to modify answer set programs such that they remain consistent given any (allowed) instance data.
We provided an analysis how contradictions in a program can arise and discussed the necessary requirements rules have to satisfy in order to prevent contradictions while keeping the degree of change to a minimum.
Based on that, we arrived at the notion of \lambdasol-extensions and we presented a strategy that gathers all conflict-resolving \lambdasol-extensions for a conflicting rule that comprise the atoms occurring in the conflicting rules.
Extending a rule by a corresponding conflict-resolving \lambdasol-extension resolves all conflicts of this rule simultaneously.
As a consequence, a conflict resolution process that uses the proposed strategy for computing conflict-resolving \lambdasol-extensions eventually yields a uniformly non-contradictory program core.

The presented approach should be seen as one of several building blocks in a larger (application) framework for updating and maintaining answer set programs.
The approach enables the knowledge expert to inspect contradictory statements in a program and update rules interactively in view of new information given by an update.
The suggestions that are generated by this algorithm can be viewed as a baseline for modifications.
As mentioned in Example~\ref{ex:run:solution}, applying a suggested \lambdasol-extension to a rule $r$ can lead to a ``stricter'' \lambdasol-extended rule $r'$.
Whether this makes sense or not, is highly dependent on the professional meaning of these rules.
Our approach allows the knowledge expert to decide whether a suggested enhancement should be applied as is or additional (manual) modifications should follow the applied suggestion.
But as shown with Proposition~\ref{prop:confresolved}, every \lambdasol-extension satisfies the minimal requirement of resolving the corresponding conflicts simultaneously and can therefore used as a reliable starting point.

This paper adresses one part of an interactive solution for conflict resolution in logic programs as proposed in \citep{Thevapalan2020}.
The running example illustrates that the proposed framework can pose a highly beneficial tool in the medical sector and other domains.
We argue that any area where experts face complex decisions based on knowledge that is not solely comprised of facts but also on practical knowledge (experience) and continually growing and changing (adapted) expertise can highly profit from such a framework to efficiently build up and maintain an adequate knowledge base.%
There still remains a number of issues that need to be adressed in order to obtain an interactive conflict resolution framework.
In the presented conflict resolution strategy, the body of one of the conflicting rules is extended.
One immediate open issue is, therefore, compiling strategies for the remaining types of program modifications, namely reduction of rule bodies, replacing body literals, removing rules as a whole and adding new rules.

As outlined Section~\ref{sec:extensions}, we see numerous ways to extend the presented approach.
In the future, we want to adapt the method of conflict resolution to resolve many-to-many conflicts in a more convenient way that generates \lambdasol-extensions based on all involved rules that at best does not require the repeated use of a resolution.
Additionally, we want to work out methods to interactively compute \lambdasol-scores (Section~\ref{subsec:lambdasol_scores}) in order to find suitable solutions more efficiently.

In each conflict resolution step during a conflict resolution process, solutions for a conflict are computed independently of other conflicts. 
In order to contemplate dependencies between conflicts, future work will also include the development of suitable heuristics to determine which rules of a conflict to modify (since a conflict-relation is symmetrical) and refine the order in which the different conflicts should be resolved.
Besides taking into account the connections between conflicts, such heuristics could also  consider individual preferences and conditions given by the expert, \eg, a set of rules that shall not be modified, rule weights \etc

Furthermore, it is conceivable that modifying the conflicting rules themselves is not always the best solution but rather other rules of the program that are responsible for the simultaneous satisfaction of conflicting rules should be altered.
Thus, one interesting future work includes the expansion of the conflict resolution approach by computing suggestions to modify those rules that are ``co-responsible'' for a potential conflict.
For the analysis of indirect causes, additional methods, like the usage of explanation graphs \citep{SchulzToni2016,PontelliSonElKhatib2009} could be useful.

Aside from contradictions, extended logic program can also become inconsistent due to incoherence.
A method to pinpoint the causes of incoherence in logic programs and strategies to reestablish coherence in a program are, therefore, also part of our current research.

The other major functionality that has yet to be developed is the actual interaction process between the expert and the framework.
These interactions must tackle two problems:
explaining the cause of the respective conflict to the user and choosing the most suitable solution for it.
Instead of burden the user with all available information, we suggest a dialogical approach where the user can reach an informed decision after a back-and-forth dialogue with the framework.
Using argumentative dialogue theories \citep{Caminada2017,ModgilCaminada2009,WaltonKrabbe1995} could enable the framework to ``communicate'' with the user where they can ascertain the underlying causes of the conflict and successively attain the most suitable suggestion.
\pagebreak
\bibliographystyle{acmtrans}
\bibliography{references}%
\label{lastpage}
\end{document}